\documentclass[12pt]{colt2019} 

\usepackage{mathrsfs,dsfont} 
\usepackage{bm,url,wasysym}
\usepackage{color}

\def\ba#1\ea{\begin{align*}#1\end{align*}} 
\def\banum#1\eanum{\begin{align}#1\end{align}} 

\newtheorem{assumption}{Assumption}

\newcommand{\drm}{\mathrm{d}}
\newcommand{\krm}{\mathrm{k}}
\newcommand{\Xc}{\mathcal{X}}
\newcommand{\Fcal}{\mathcal{F}}
\newcommand{\Gcal}{\mathcal{G}}
\newcommand{\Hcal}{\mathcal{H}}
\newcommand{\Tcal}{\mathcal{T}}
\newcommand{\Qcal}{\mathcal{Q}}
\newcommand{\Bcal}{\mathcal{B}}
\newcommand{\Pb}{\mathbb{P}}
\newcommand{\Rb}{\mathbb{R}}
\newcommand{\Eb}{\mathbb{E}}

\makeatletter
\newcommand{\biggg}[1]{{\hbox{$\left#1\vbox to 20.5pt{}\right.\n@space$}}}
\newcommand{\Biggg}[1]{{\hbox{$\left#1\vbox to 23.5pt{}\right.\n@space$}}}
\newcommand{\bigggg}[1]{{\hbox{$\left#1\vbox to 26.5pt{}\right.\n@space$}}}
\newcommand{\Bigggg}[1]{{\hbox{$\left#1\vbox to 29.5pt{}\right.\n@space$}}}
\newcommand{\biggggg}[1]{{\hbox{$\left#1\vbox to 32.5pt{}\right.\n@space$}}}
\newcommand{\Biggggg}[1]{{\hbox{$\left#1\vbox to 35.5pt{}\right.\n@space$}}}
\newcommand{\bigggggg}[1]{{\hbox{$\left#1\vbox to 38.5pt{}\right.\n@space$}}}
\newcommand{\Bigggggg}[1]{{\hbox{$\left#1\vbox to 41.5pt{}\right.\n@space$}}}
\makeatother
\newcommand{\bigggl}{\mathopen\biggg}

\newcommand{\bigggr}{\mathclose\biggg}

\title[Fast learning rate of deep learning without scale invariance]
{Fast generalization error bound of deep learning\\ without scale invariance of activation functions}
\usepackage{times}

\coltauthor{%
 \Name{Yoshikazu Terada}$^\ast$
\Email{terada@sigmath.es.osaka-ua.c.jp}\and\\
\Name{Ryoma Hirose} \Email{r-hirose@sigmath.es.osaka-u.ac.jp}\\
 \addr Graduate School of Engineering Science, Osaka University\\
1-3 Machikaneyama-cho, Toyonaka, Osaka 560-8531, Japan.\\
$^\ast$Jointly affiliated at 
RIKEN Center for Advanced Intelligence Project (AIP),\\
1-4-1 Nihonbashi, Chuo-ku, Tokyo 103-0027, Japan.
}
\begin{document}
\maketitle

\begin{abstract}%
In theoretical analysis of deep learning, 
discovering which features of deep learning lead to good performance is an important task.
In this paper, 
using the framework for analyzing the generalization error developed in \cite{Suzuki18}, 
we derive a fast learning rate for deep neural networks with more general activation functions.
In \cite{Suzuki18}, 
assuming the scale invariance of activation functions, 
the tight generalization error bound of deep learning was derived. 
They mention that the scale invariance of the activation function is essential to derive tight error bounds.
Whereas the rectified linear unit (ReLU; \citealp{NairHinton10}) satisfies the scale invariance, 
the other famous activation functions including the sigmoid and the hyperbolic tangent functions,
and the exponential linear unit (ELU; \citealp{ClevertEtAl16}) does not satisfy this condition.
The existing analysis indicates a possibility
that a deep learning with the non scale invariant activations may have a slower convergence rate of $O(1/\sqrt{n})$
when one with the scale invariant activations can reach a rate faster than $O(1/\sqrt{n})$.
In this paper, without the scale invariance of activation functions, 
we derive the tight generalization error bound which is essentially the same as that of \cite{Suzuki18}.
From this result, at least in the framework of \cite{Suzuki18}, 
it is shown that the scale invariance of the activation functions is not essential to get the fast rate of convergence. 
Simultaneously, it is also shown that
the theoretical framework proposed by \cite{Suzuki18} can be widely applied for analysis of deep learning with general activation functions.
\end{abstract}


%
\section{Introduction}

For various application tasks including computer vision and natural language processing,
deep learning has achieved great performance and has made a significant impact on the related fields.
It is important to provide theoretical answers for the question of why deep learning provides great performance.

One of the important features of deep learning is its universal approximation ability and high expressive power.
In 1989, \cite{Cybenko89}, \cite{HornikEtAl89}, and \cite{Funahashi89} independently proved that the shallow neural network with one hidden layer has the universal approximation ability. 
That is, for large enough number of hidden nodes, every continuous function on the bounded subset in $\mathbb{R}^d$ can be approximated arbitrarily-well, 
uniformly over the bounded set, by neural networks with one hidden layer.
Moreover, 
\cite{Murata96}, \cite{Candes99}, and \cite{SonodaMurata17} consider the neural network 
as the discretization of its integral representation, and clarify its approximation ability.
Recent years, in terms of expressive power, 
there are many theoretical answers to why the deep neural networks are preferred to shallow ones.
For example, 
\cite{EldanShamir16} show that there exists a simple function on $\Rb^d$, 
expressible by a small $4$-layer neural networks with polynomial order widths,
which cannot be arbitrarily well approximated by any shallow network with one hidden layer,
unless its width is exponential in the input dimension $d$.
It has been shown that
deep neural networks are more expressive than shallow ones of comparable size
\citep{BengioDelalleau11,BianchiniScarselli14, MontufarEtAl14,CohenEtAl16,PooleEtAl16, Yarotsky17, PetersenVoigtlaender18}

Another important point in theoretical analysis of machine learning algorithms
is to derive its generalization error.
From the discussion of no-free-lunch theorem or the slow rates of convergence (see, Chapter 7 in \cite{DevroyeEtAl96}),
we know that universally good algorithms do not exist.
Thus, it is important to clarify when the fast learning rate can be achieved, 
or what the fastest achievable rate is for the given class of distributions.
For example, in \cite{KoltchinskiiPanchenko02}, \cite{NeyshaburEtAl15}, and \cite{SunEtAl15},  
generalization bounds for neural networks were derived by evaluating the Rademacher complexity.
For deep neural networks with rectified linear units (ReLU) \citep{NairHinton10}, 
\cite{Schmidt-Hieberi18}, \cite{ImaizumiFukumizu18} , and \cite{Suzuki19} give deep insights into
why deep neural networks outperform other existing methods such as the kernel ridge regression in practice through its theoretical analysis.
Moreover, in recent years, for deep neural networks with ReLU activation functions,
there are several theoretical answers for why deep learning can avoid the curse of dimensionality (\citealp{BarrobKlusowski18,Suzuki19}).
In \cite{Suzuki18}, focusing on the integral representation in \cite{SonodaMurata17},
a new important theoretical framework to analyze the generalization error of deep learning is developed.
Using this framework, the approximation error of deep neural networks with scale invariant activation functions, 
which can be interpreted as the discretization of its integral representation,
can be evaluated by the degree of freedom of the reproducing kernel Hilbert space (RKHS) induced in each layer, 
which is introduced in \cite{Bach17b}. 
Moreover, through the analysis of the generalization error,
we can see bias-variance trade-off in terms of the number of parameters in the deep neural network.
From these results, 
we can determine the optimal widths (the number of nodes) of the internal layers, and 
can derive the optimal convergence rate faster than $O(1/\sqrt{n})$.

In this paper, we focus on the fast rate of convergence for deep learning.
Most theoretical studies deal with deep neural networks with the ReLU activation functions.
\cite{Suzuki18} considers a wider class of activation functions that satisfy the scale invariance.
The ReLU and leaky ReLU (LReLU; \citealp{MaasEtAl13}) functions satisfy the scale invariance
whereas other famous activation functions such as the sigmoid and the hyperbolic tangent activation functions do not satisfy this condition.
\cite{Suzuki17, Suzuki18} mentioned that the scale invariance of the activation function is essential to derive tight error bounds
and also indicated that we only have a much looser bound for the approximation error without the scale invariance.
Here, we note that, in \cite{Suzuki18}, 
the generalization error of the deep neural networks without the scale invariance is not derived,
whereas the approximation error is derived in \cite{Suzuki17}.
If we believe that the variance term of the generalization error is the same even for the deep neural networks with the non scale invariant activation function, 
the existing results indicate that 
a deep learning with the non scale invariant activations may have a slower convergence rate of $O(1/\sqrt{n})$
even when one with the scale invariant activations can reach a convergence rate faster than $O(1/\sqrt{n})$.
In contrast, recently, \cite{ClevertEtAl16} propose a new important activation function, called exponential linear unit (ELU).
Although the ELU function is not scale invariant, 
the ELU function does not only speed up learning in deep neural networks 
but also leads to higher (or comparable) classification accuracies compared to other activation functions 
including the ReLU and LReLU functions in many practical situations.
Thus, it is natural to ask an important question ``Is the scale invariance of the activation functions essential to get the fast rate of convergence for deep learning?"
In this paper, we will show that, even without the scale invariant assumption on the activation functions,
the tight upper bound in \cite{Suzuki18} can be derived in the same theoretical framework.
Whereas the overall flow of the proofs of our results is the same as \cite{Suzuki18}, 
the detail parts of the proofs are different, and the proofs in this paper are not trivial.
From our improved results, 
we can see that the theoretical framework developed in \cite{Suzuki18} can be widely applied for deep neural networks with general activation functions, 
and we can provide the unified practical approach described in \cite{Suzuki18} for the determination of the widths on the internal layers.

This paper is organized as follows.
In Section~\ref{sec:2}, we set up notation and terminology used in \cite{Suzuki18}.
In Section~\ref{sec:3}, we introduce some assumptions and derive a tight upper bound of the approximation error 
under the no assumption of the scale invariance of the activation functions.
Section~\ref{sec:4} provides the upper bound of the generalization error for deep learning without the scale invariance.
Under no assumption of the scale invariance, 
we also discuss the difference between our tight error bound and the looser error bound which can be derived from the looser approximation error bound of \cite{Suzuki17}
in Section~\ref{sec:4}.

%

%
\section{Preliminaries}\label{sec:2}

Since we employ the theoretical framework developed in \cite{Suzuki17, Suzuki18} to derive the generalization error for general feedforward deep neural networks,
we follow the notation used in \cite{Suzuki18}.
We assume in this paper that the data $D_n=(X_i,Y_i)_{i=1}^n$ is a sequence of independently identically generated from the following model:
\[
Y_i = f^{o}(X_i) + \xi_i\quad(i = 1,\dots,n),
\]
where $(\xi_i)_{i=1}^n$ is an i.i.d. sample from Gaussian distribution $N(0,\sigma^2)$ with mean $0$ and variance $\sigma^2$, and 
$(x_i)_{i=1}^n$ is an i.i.d. sample from the distribution $P_X$ on $\mathbb{R}^{d_x}$ whose support is compact. 
Here, we consider the regression problem in which we estimate $f^o$ from data $D_n$ by deep neural networks.
Using the ridgelet analysis, 
\cite{SonodaMurata17} shows that, for any function $f\in L_1(\Rb^{d_X})$ which has an integrable Fourier transform, 
there exists the following integral form with an activation function $\eta$ such as ReLU:
\banum
f(x) = \int h(w,b)\eta(w^Tx + b)\,\drm w\drm b + b^{(2)}, \label{eq:integral-rep}
\eanum
where $(w,b)\in \Rb^{d_X}\times \Rb$, $b^{(2)}\in \Rb$, and $h : \Rb^{d_X}\times \Rb \rightarrow \Rb$.
For a $d$-dimensional vector $x=(x_1,\dots,x_d)\in\Rb^d$, write
$\eta(x) = (\eta(x_j))_{j=1}^d$.
The shallow neural network with one internal layer is represented as
\[
f(x) 
= W^{(2)}\eta( W^{(1)}x+b^{(1)}) + b^{(2)}
=\sum_{i=1}^{m_2}w_{i}^{(2)}\eta\left( w_i^Tx + b_i\right) + b^{(2)},
\]
where $m_2$ is the number of nodes in the internal layer, 
$W^{(2)} \in \mathbb{R}^{1\times m_2},\;W^{(1)}=(w_1,\dots,w_{m_2})^T\in \mathbb{R}^{m_2\times d_x}$, and $b^{(1)}\in \mathbb{R}^{m_2}$, $b^{(2)}\in \mathbb{R}$.
Thus, the shallow neural network can be considered as the discretization of the integral form $(\ref{eq:integral-rep})$.

It is known that deep neural networks are more expressive than shallow ones of comparable size (see, e.g., \citealp{EldanShamir16}).
Thus, we will consider an integral representation for a deeper neural network.
To construct the integral form, we define the feature space of the $\ell$-th layer as a probability space.
Let $(\mathcal{T}_\ell, \mathcal{B}_\ell, \mathcal{Q}_\ell)$ be a probability space 
where $\mathcal{T}_\ell$ is a Polish space, $\Bcal_\ell$ is its Borel algebra, and $\Qcal_\ell$ is a probability measure on $(\Tcal_\ell, \Bcal_\ell)$.
If the $\ell$-th internal layer is not continuous and has $d_\ell$ nodes, then simply $\Tcal_\ell = \{1,\dots,d_\ell\}$.
Since the input space is $d_x$-dimensional, $\Tcal_1=\{1,\dots,d_x\}$.
In our setting, the feature space of the output layer is singleton $\Tcal_{L+1}=\{1\}$.
In the integral form $(\ref{eq:integral-rep})$,
the feature space on the second layer is corresponding to the continuous space $\Tcal_2=\{(w,b)\in \Rb^{d_x}\times \Rb\}$.
The map $f_\ell^o:L_2(\Qcal_\ell)\rightarrow L_2(\Qcal_{\ell+1})$ on the $\ell$-th layer is defined by
\[
f_\ell^o[g](\tau) = \int_{\Tcal_\ell} h_\ell^o(\tau,w)\eta(g(w))\drm \Qcal_\ell(w) + b_\ell^o(\tau),
\]
where $h_\ell^o(\tau,w)$ is the weight of the feature $w$ for the output $\tau$, 
and $h_\ell^o\in L_2(\Qcal_{\ell+1}\times \Qcal_\ell)$ and $h_\ell^o(\tau,\cdot)\in L_2(\Qcal_\ell)$ for all $\tau \in \Tcal_{\ell+1}$.
Now, we construct the integral representation of deep neural network as follows:
\banum\label{eq:deep-integral-rep}
f^o(x)=f_L^o\circ f_{L-1}^o\circ \dots \circ f_1^o(x),
\eanum
where $f_\ell^o\circ F(x)$ denotes $f_\ell^o[F(x)](\cdot) \in L_2(\Qcal_{\ell+1})$ for $F(x)(\cdot) \in L_2(\Qcal_\ell)$,
\ba
f_1^o[x](\tau) &= \sum_{j=1}^{d_x} h_1^o(\tau,j)x_j\Qcal_1(j) + b_1^o(\tau),\;
f_L^o[g](1) = \int_{\Tcal_L}h_L^o(1,w)\eta(g(w))\drm \Qcal_L(w) + b_L^o, \text{ and}\nonumber\\
f_\ell^o[g](\tau) &= \int_{\Tcal_\ell} h_\ell^o(\tau,w)\eta(g(w))\drm \Qcal_\ell(w) + b_\ell^o(\tau)\quad(\ell = 2, \dots, L-1).
\ea
We shall assume that the true function $f^o$ has this integral representation $(\ref{eq:deep-integral-rep})$.
Since the neural network with one internal layer is a universal approximator,
the deep neural network model $(\ref{eq:deep-integral-rep})$ can be also a universal approximator.
%

%
\section{Finite approximation error bound of the integral representation}\label{sec:3}

To estimate $f^o$ from data $D_n$, 
it is necessary to discretize the integrals in the form $(\ref{eq:deep-integral-rep})$ by finite sums.
The usual deep neural network model can be interpreted as the discrete approximation of the integral form.
Here, under no assumption of the scale invariance of the activation functions,
we derive a tight upper bound of the approximation error of the discretization by employing the notions in the kernel method.

According to \cite{Suzuki18}, 
we construct the reproducing kernel Hilbert space (RKHS) for each layer.
Let $F_\ell^o(x,\tau):=(f_\ell^o\circ \dots f_1^o(x))(\tau)$ be the output of the $\ell$-th layer.
For $\ell\ge 2$, we define the kernel $\krm_\ell : \Rb^{d_x}\times \Rb^{d_x}\rightarrow \Rb$ corresponding to the $\ell$-th layer as
\[
\krm_\ell(x,x^\prime):=
\int_{\Tcal_\ell}\eta(F_{\ell-1}^o(x,\tau))\eta(F_{\ell-1}^o(x^\prime,\tau))\drm \Qcal_\ell(\tau).
\]
By the Moore-Aronszajn theorem \citep{Aronszajn50}, 
there exists a unique RKHS $\Hcal_\ell$ corresponding to the kernel $\krm_\ell$.
We consider the following bounded linear operator $S_\ell :  L_2(\Qcal_\ell) \rightarrow  L_2(P_X)$:
\[
(S_\ell h)(x) = \int_{\Tcal_\ell} h(\tau)\eta(F_{\ell-1}^o(x,\tau))\drm \Qcal_\ell(\tau).
\]
\cite{Bach17a, Bach17b} show that the image of $S_\ell$ is equivalent to RKHS $\Hcal_\ell$ and that
the norm $\|g\|_{\Hcal_\ell}^2$ for $g \in \Hcal_\ell$ is equal to the minimum of $\|h\|_{L_2(\Qcal_\ell)}^2$ over all $h$ such that $S_\ell h = g$.
Thus, the function $x\mapsto \int_{\Tcal_\ell} h_\ell^o(\tau,w)\eta(F_{\ell-1}^o(x,w))\,\drm\Qcal_\ell(w)$ is in the RKHS $\Hcal_\ell$, and
its RKHS norm is equal to the norm of weight function of the internal layer $\|h_\ell^o(\tau,\cdot)\|_{L_2(\Qcal_\ell)}$.

Now, we define the complexity of the RKHS introduced in \cite{Bach17b}.
We will consider the following integral operator $T_\ell :  L_2(P_X) \rightarrow  L_2(P_X)$:
\[
(T_\ell g)(x) = \int_{\Xc} \krm_\ell(\cdot, x)g(x)\drm P_X(x).
\]
By Mercer's theorem, we obtain the following decomposition:
\[
\krm_\ell(x,x^\prime)=\sum_{j=1}^\infty \mu_j^{(\ell)} \phi_j^{(\ell)}(x)\phi_j^{(\ell)}(x^\prime),
\]
where $(\mu_j^{(\ell)})_{j=1}^\infty$ is the sequence of the eigenvalues of $T_\ell$ ordered in decreasing order, and 
$(\phi_j^{(\ell)})_{j=1}^\infty$ is the sequence of the corresponding eigenfunctions, 
which forms an orthonormal system in $L_2(P_X)$.
For $\lambda>0$, the degree of freedom of the RKHS $\Hcal_\ell$ is defined by
\banum\label{eq:5}
N_\ell(\lambda)=\mathrm{tr}\{(T_\ell+\lambda I)^{-1}T_\ell\}
=\sum_{j=1}^\infty \frac{\mu_j^{(\ell)}}{\mu_j^{(\ell)} + \lambda},
\eanum
which is analogous to a traditional quantity in the analysis of least-squares regression.
Through the discussion of Section 4.2 in \cite{Bach17b}, 
we can intuitively consider that this complexity measures an effective dimension of $\Hcal_\ell$.
It is worth noting that $N_\ell(\lambda)$ is monotonically decreasing with respect to $\lambda$.
The following lemma, which is the direct consequence from Proposition 1 in \cite{Bach17b},
plays a key role in the approximation error analysis.
\begin{proposition}\label{prop:1}
For any $\lambda>0$ and any $1/2>\delta>0$, if
\[
m_\ell \ge 5N_\ell(\lambda)\log(16N_\ell(\lambda)/\delta)\quad(\ell = 2,\dots,L),
\]
then there exists $v_1^{(\ell)},\dots,v_{m_\ell}^{(\ell)}\in \Tcal_\ell$ and $w_1^{(\ell)},\dots,w_{m_\ell}^{(\ell)}>0$ such that
\[
\sup_{\|f\|_{\Hcal_\ell}\le R}\inf_{\beta^{(\ell)}\in \Rb^{m_\ell}:\|\beta^{(\ell)}\|_2^2\le \frac{4R^2}{m_\ell}}
\left\|f - \sum_{j=1}^{m_\ell}\beta_j^{(\ell)}w_j^{(\ell)}\eta(F_{\ell-1}(\cdot,v_j^{(\ell)})) \right\|_{L_2(P_X)}^2
\le 4\lambda R^2,
\]
and
\[
\frac{1}{m_{\ell}}\sum_{j=1}^{m_\ell} {w_j^{(\ell)}}^2 \le (1-2\delta)^{-1}.
\]
\end{proposition}
\begin{proof}
The proof can be founded in the supplementary material of \cite{Suzuki18}.
\end{proof}

At first, we will make some assumptions.
We assume that the true function $f^o$ satisfies the following norm condition.
\begin{assumption}\label{assmp:1}
For all $\ell$, $\;h_\ell^o$ and $b_\ell^o$ satisfy the following, respectively:
\[
\|h_\ell^o(\tau,\cdot)\|_{L_2(\Qcal_\ell)} \le R,\quad
|b_{\ell}^o(\tau)|\le R_b\quad(\forall \tau \in \Tcal_\ell).
\]
\end{assumption}
For the activation functions, we {\it do not} assume the scale invariance.
\begin{assumption}\label{assmp:2}
The activation function is $1$-Lipschitz continuous:
\[
\forall x,x^\prime \in \Rb;\;|\eta(x)- \eta(x^\prime)|\le |x-x^\prime|.
\]
\end{assumption}
Note that usual activation functions including the sigmoid function, ReLU, and ELU satisfy this assumption. 

Finally, we assume that the support of the input distribution $P_X$ is compact.
\begin{assumption}\label{assmp:3}
Let $\mathrm{supp}(P_X)$ denote the support of $P_X$, and assume
\[
\exists D_x>0;\;\forall x \in \mathrm{supp}(P_X);\;\|x\|_\infty:=\max_{1\le i \le d_x}|x_i|\le D_x.
\]
\end{assumption}

Based on Proposition~\ref{prop:1}, 
we will consider the finite dimensional approximation model $f^\ast$ of $f^o$．
Let us denote by $m_\ell$ the number of nodes in the $\ell$-th internal layer, and define the following model:
\ba
f_1^\ast(x) &= W^{(1)}x+b^{(1)}, \quad f_\ell^\ast(g)=W^{(\ell)}\eta(g) + b^{(\ell)}\;(g\in \Rb^{m_\ell},\;\ell=2,\dots,L),\\
f^\ast(x) &= f_L^\ast\circ f_{L-1}^\ast\circ \dots \circ f_1^\ast(x),
\ea
where $W^{(\ell)}\in \Rb^{m_{\ell+1}\times m_{\ell}}$ and $b^{(\ell)}\in \Rb^{m_{\ell+1}}$.
Set $c_0=4,\;c_1=4,\;c_\delta = (1-\delta)^{-1}$.
For short, denote $\hat{c}_{\delta}=c_1c_\delta$. 
\begin{theorem}\label{theorem:approx-error} (Approximation error bound without the scale invariance)\\
Let Assumptions \ref{assmp:1} to \ref{assmp:3} hold
For any $\delta\in(0,1)$ and given $\lambda_\ell>0$, suppose that
\[
m_\ell \ge 5N_\ell(\lambda_\ell)\log(32N_\ell(\lambda_\ell)/\delta)\quad(\ell = 2,\dots,L).
\]
Then, there exists $W^{(\ell)}\in \Rb^{m_{\ell+1}\times m_{\ell}}$ and $b^{(\ell)}\in \Rb^{m_{\ell+1}}\;(\ell = 1,\dots,L)$ such that
\banum\label{eq:norm-condition}
\|W^{(\ell)}\|_\infty\le \sqrt{\hat{c}_{\delta}} R,\quad\|b^{(\ell)}\|_{\max}\le R_b\quad(\ell=1,\dots,L)
\eanum
and
\banum\label{eq:approx-error}
\|f^o-f^\ast\|_{L_2(P_X)}\le \sum_{\ell=2}^L 2\sqrt{\hat{c}_{\delta}^{L-\ell}}R^{L-\ell+1}\sqrt{\lambda_\ell},
\eanum
where the matrix norm $\|\cdot\|_\infty$ is defined by
\[
\|A\|_\infty := \max_{i=1,\dots,p}\sum_{j=1}^q |a_{ij}| \;\text{ for }\;A=(a_{ij})_{p\times q}\in \Rb^{p\times q}.
\]
\end{theorem}
\begin{proof}
See Appendix~\ref{sec:proof-theorem:approx-error}.
\end{proof}
In the above theorem, due to the non scale invariance, 
the magnitudes of $W^{(\ell)}$ and $b^{\ell}$ are different from \cite{Suzuki18}
whereas the upper bound is the same.
Using the derivation in \cite{Suzuki18},
even under the same condition $(\ref{eq:norm-condition})$,
we obtain only the following upper bound as described in \cite{Suzuki17}:
\[
\|f^\ast - f^o\|_{L_2(P_X)}\le \sum_{\ell = 2}^L2\sqrt{m_{\ell + 1}\hat{c}_\delta^{L-\ell}}R^{L-\ell + 1}\sqrt{\lambda_\ell}.
\]
This looser bound depends on the dimensions $(m_\ell)_{\ell=1}^L$ of the internal layers, which could be very large for small $\lambda_\ell$.
\cite{Suzuki17} mentioned that 
this fact (the loose bound without the scale invariance) supports the practical success of using the ReLU activation functions.
In contrast, from Theorem~\ref{theorem:approx-error}, 
both \cite{Suzuki18} and our result seem to support the practical success of using the deep learning structure.
It is worth noting that
the detail part of the derivation of the upper bound will be important to derive
the tight upper bound of the generalization error (especially for evaluating the covering number).

In the proof of Theorem~\ref{theorem:approx-error},
the big difference from \cite{Suzuki18} is the construction of the candidate of the discretization $f^\ast$ 
based on Proposition~\ref{prop:1}.
In the proof of the corresponding result in \cite{Suzuki18}, 
the parameters of a candidate for the discretization $f^\ast$ are taken as follows:
\ba
W^{(1)}&=\frac{1}{\sqrt{m_2}}\left(\Qcal_1(j)w_{i}^{(2)}h_1^o(v_{i}^{(2)}, j) \right)_{i,j}\in \Rb^{m_2\times d_x},\quad
W^{(L)}=\sqrt{m_L}\beta^{(L)}\in \Rb^{1\times m_L},\\
W^{(\ell)} &= \sqrt{\frac{m_\ell}{m_{\ell+1}}}\left( \beta_{ij}^{(\ell)}w_i^{\ell+1} \right)_{i,j} \in \Rb^{m_{\ell+1}\times m_\ell}
\quad(\ell = 2,\dots,L-1),\\
b^{(1)} &= \frac{1}{\sqrt{m_2}} \left( w_1^{(2)}b_1^o(1),\dots,w_{m_2}^{(2)}b_1^o(m_2) \right)^T \in \Rb^{m_2}, 
\quad b^{(L)} = b_L^o(1)\in \Rb, \quad\text{and}\\
b^{(\ell)}&= \frac{1}{\sqrt{m_{\ell+1}}} \left( w_1^{(\ell+1)}b_\ell^o(v^{(\ell+1)}),\dots, w_{m_{\ell+1}}^{(\ell+1)}b_\ell^o(v_{m_{\ell+1}^{(\ell+1)}}) \right)^T \in \Rb^{m_{\ell+1}}
\quad(\ell = 2,\dots,L-1).
\ea
We can see that the derivation takes the advantage of the scale invariance in an efficient way.
In contrast, in our proof, from Proposition~\ref{prop:1}, the parameters are naturally taken as follows:
\ba
W^{(1)} &=  \left( h_1^o(v_i^{(2)},j)\Qcal_1(j) \right)_{i,j} \in \Rb^{m_2\times d_X},\quad
W^{(L)} = ({\beta^{(L)}}\odot w^{(L)})^T\in \Rb^{1\times m_L},\\
W^{(\ell)} &= \left( \beta_{ij}^{(\ell)}w_i^{(\ell)} \right)_{i,j}\in \Rb^{m_{\ell+1}\times m_\ell}\quad(\ell = 2,\dots,L-1),\\
b^{(1)} &= \left(b_1^o(v_1^{(2)}), \dots ,b_1^o(v_{m_2}^{(2)}) \right)^T \in \Rb^{m_2},\quad b^{(L)}=b_L^o(1)\in \Rb,\quad\text{and}\\
b^{(\ell)}& =\left(b_\ell^{o}(v_{1}^{(\ell+1)}),\dots,b_\ell^{o}(v_{m_{\ell+1}}^{(\ell+1)}) \right)^T \in \Rb^{m_{\ell+1}}\quad(\ell = 2,\dots,L-1),
\ea
where $\odot$ denotes the Hadamard product.

Hereafter, fix $\delta>0$. For simplicity of notation, 
write $\bar{R}:= \sqrt{\hat{c}_{\delta}}R$.
According to Theorem~\ref{theorem:approx-error},
we will consider the following class $\Fcal$ of the finite dimensional functions:
\[
\Fcal=\{f(x) = (W^{(L)}\eta(\cdot) + b^{(L)})\circ\cdots\circ (W^{(1)}x + b^{(1)})\mid
 \|W^{(\ell)}\|_\infty\le \bar{R},\; \|b^{(\ell)}\|_{\max} \le R_b\;(\ell = 1,\dots,L)\}.
\]
Note that this candidate class $\Fcal$ is different from that of \cite{Suzuki18}.
In fact, the following class is considered in \cite{Suzuki18}:
\[
\Fcal^\prime
=
\{f(x) = (W^{(L)}\eta(\cdot) + b^{(L)})\circ\cdots\circ (W^{(1)}x + b^{(1)})\mid
 \|W^{(\ell)}\|_F\le \bar{R},\; \|b^{(\ell)}\|_2 \le c_\delta R_b\;(\ell = 1,\dots,L)\},
\]
where $\|W^{(\ell)}\|_F$ is the Frobenius norm of the matrix $W^{(\ell)}$.

Now, we evaluate the magnitudes of the true function $f^o$ and considered functions $f\in \Fcal$ by the infinity norm.
\begin{lemma}\label{lemma:bound}
Under Assumptions 1 to 3, the $L_\infty$-norms of $f^o$ and $f\in \Fcal$ are bounded as follows:
\ba
\|f^o\|_\infty &\le R^LD_x + \sum_{\ell = 1}^LR^{L-\ell}(R_b+c_\eta),\quad
\|f\|_\infty \le \bar{R}^{L}D_x + \sum_{\ell=1}^{L}\bar{R}^{L-\ell}(c_\eta+R_b),
\ea
where $c_\eta:=\eta(0)$.
\end{lemma}
\begin{proof}
See Appendix~\ref{sec:proof-lemma:bound}.
\end{proof}

Whereas we do not assume the scale invariance of the activation functions,
the upper bound in the above lemma is essentially the same as that of \cite{Suzuki18}.
From now on, write
\[
\hat{R}_\infty := \bar{R}^{L}D_x + \sum_{\ell=1}^{L}\bar{R}^{L-\ell}(c_\eta+R_b).
\]

%

%
\section{Convergence rate for the empirical risk minimizer without the scale invariance}\label{sec:4}

Here, we will consider the following empirical risk minimizer:
\[
\hat{f}=\mathop{\arg\min}_{f\in \Fcal} \sum_{i=1}^n\{y_i - f(x_i)\}^2.
\]
Since $\eta$ is continuous and the parameter space corresponding to $\Fcal$ is compact, 
we can ensure the existence of $\hat{f}$.
As with \cite{Suzuki18}, for theoretical simplicity,
we assume that $\hat{f}$ is the exact minimizer whereas we can take $\hat{f}$ as an approximated minimizer.
The flow of the proof of the main theorem described later is the same as that of \cite{Suzuki18}.
The theorem can be proved through the evaluation of the covering number of $\Fcal$ 
and the local Rademacher complexity technique (see, e.g., \cite{Koltchinskii06}).
In our proof of the main theorem, 
the different point from that of \cite{Suzuki18} is the evaluation of the covering number of $\Fcal$.

First, we state some fundamental lemmas for deriving the generalization error of the empirical risk minimizer 
without the scale invariance.
The following lemma is important to evaluate the covering number of $\Fcal$.
\begin{lemma}\label{lemma:difference-bound}
Let Assumptions 1 to 3 hold.
Let $f,f^\prime \in \Fcal$ be two functions with parameters $\{(W^{(\ell)},b^{(\ell)})\}_{\ell=1}^L$ and $\{({W^\prime}^{(\ell)},{b^\prime}^{(\ell)})\}_{\ell=1}^L$, respectively.
Let $\epsilon >0$.
If $\|W^{(\ell)} - {W^\prime}^{(\ell)}\|_\infty<\epsilon$ and $\|b^{(\ell)}-{b^\prime}^{(\ell)}\|_{\max}<\epsilon$ for $\ell = 1,\dots,L$, 
then
\[
\|f-f^\prime\|_\infty
\le
\epsilon 
\left\{
L\bar{R}^{L-1}\left[ D_x +  L(c_\eta+R_b) \right]
+  \sum_{\ell = 1}^{L}\bar{R}^{L-\ell}
\right\}.
\]
\end{lemma}
\begin{proof}
See Appendix~\ref{sec:proof-lemma:difference-bound}.
\end{proof}
We will write the term in the bracket of the upper bound $\hat{G}$ for short.
By Lemma~\ref{lemma:difference-bound},
we can evaluate the $\epsilon$-covering number $N(\epsilon, \Fcal, \|\cdot\|_\infty)$ of $\Fcal$ without using the scale invariance.
\begin{proposition}\label{prop:covering-number}
Let Assumptions 1 to 3 hold. Then,
\ba
\log N(\epsilon, \Fcal, \|\cdot\|_\infty)
&\le 
\log\left\{ 1 + \frac{2\hat{G}\max\{\bar{R},R_b\}}{\epsilon} \right\}
\left\{ \sum_{\ell = 1}^L (m_{\ell+1}+1) m_\ell \right\}
\ea
\end{proposition}
\begin{proof}
See Appendix~\ref{sec:proof-prop:covering-number}.
\end{proof}
It is worth noting that the upper bound is very similar to that of \cite{Suzuki18} whereas $\Fcal$ is different.
Using the derivation of \cite{Suzuki18}, under no assumption of the scale invariance,
we may get a much looser bound even for the covering number.

Based on these results, without the scale invariance of the activation functions, 
we can derive the tight upper bound of the generalization error of the empirical risk minimizer as follows.
\begin{theorem}\label{theorem:main}
Let Assumptions 1 to 3 hold.
For any $\delta\in(0,1)$ and $\lambda_\ell>0\;(\ell=2,\dots,L)$, suppose that
\banum
m_\ell \ge 5N_\ell(\lambda_\ell)\log(32N_\ell(\lambda_\ell)/\delta)\quad(\ell = 2,\dots,L).\label{eq:constraint}
\eanum
Then, there exists a universal constant $C$ such that, for any $r>0$ and for any $\tilde{r}\in(1,2]$,
with probability at least
\[
1-\exp\left(-\frac{n\hat{\delta}_{1,n}^2(\tilde{r}-1)^2}{11} \right)-2\exp(-r),
\]
we have
\banum
\|\hat{f}-f^o\|_{L_2(P_X)}^2
&\le 
C
\left\{\tilde{r}\hat{\delta}_{1,n}^2+
(\sigma^2+ \hat{R}_\infty^2)\hat{\delta}_{2,n}^2 + \frac{\sigma^2 
+ \hat{R}_\infty^2}{n}\left[ \log_+\left(\frac{\sqrt{n}}{\min\{1,\sigma/\hat{R}_\infty\}}\right)  + r\right]
\right\},\label{eq:tight-bound}
\eanum 
where $\log_+(x)= \max\{1,\log(x)\}$,
\ba
\hat{\delta}_{1,n} &:= \sum_{\ell=2}^L 2\sqrt{\hat{c}_{\delta}^{L-\ell}}R^{L-\ell+1}\sqrt{\lambda_\ell},\quad{and}\\
\hat{\delta}_{2,n} &:=
\sqrt{\frac{\sum_{\ell = 1}^L m_{\ell+1} m_\ell}{n}
\log_+\left(1+\frac{\sqrt{n}\hat{G}\max\{\bar{R},R_b\}}{\min\{\sigma, \hat{R}_\infty\} \sqrt{\sum_{\ell = 1}^L m_{\ell+1} m_\ell}}\right)}.
\ea
\end{theorem}
\begin{proof}
From Theorem~\ref{theorem:approx-error} and Proposition~\ref{prop:covering-number}, under no assumption of the scale invariance,
we can derive the tight upper bound of the generalization error in the same way as \cite{Suzuki18}.
For the sake of completeness and self-containedness, 
we provide the proof in Appendix~\ref{sec:proof-theorem:main}.
\end{proof}
This generalization error bound is essentially the same as that of \cite{Suzuki18} 
although we do not assume the scale invariance of the activation functions.
Since the third term in the bracket of the upper bound is smaller than the first two terms, 
the generalization error can be simply represented by
\[
\|\hat{f} - f^o\|_{L_2(P_X)}^2 = O_p(\hat{\delta}_{1,n}^2 + \hat{\delta}_{2,n}^2).
\]
Intuitively, 
the term $\hat{\delta}_{1,n} $ can be considered as the bias term 
which is induced in approximation of $f^o$ by the finite dimensional model $\Fcal$.
Moreover, the term $\hat{\delta}_{2,n}$ represents the variance term, that is,
the deviation of the estimator in the finite dimensional model $\Fcal$.
According to Theorem~\ref{theorem:approx-error}, 
large widths $m_\ell$ fo the internal layers are required to obtain a small value of $\hat{\delta}_{1,n}$.
However, large widths $m_\ell$ lead to the increase of the variance term $\hat{\delta}_{2,n}$.
Thus, we can see the bias-variance trade-off in the generalization error bound.

In the existing approximation error analysis in \cite{Suzuki17, Suzuki18} 
with the tight evaluation of the covering number of $\Fcal$ in Proposition~\ref{prop:covering-number},
under no assumption of the scale invariance, 
we obtain the following looser bound:
\banum
\|\hat{f} - f^o\|_{L_2(P_X)}^2 = O_p(\hat{\Delta}_{1,n}^2 + \hat{\delta}_{2,n}^2),\label{eq:looser-bound}
\eanum
where
\[
\hat{\Delta}_{1,n}:=\sum_{\ell = 2}^L2\sqrt{m_{\ell + 1}\hat{c}_\delta^{L-\ell}}R^{L-\ell + 1}\sqrt{\lambda_\ell}.
\]
Through the example of the generalization error bound, 
we will see the essential difference between the looser bound $(\ref{eq:looser-bound})$ and our tight bound $(\ref{eq:tight-bound})$ in Theorem~\ref{theorem:main}.
Suppose that $\sigma,\;\hat{R}_\infty$, and $\bar{R}^L$ are of constant order.
Then, we can rewritten the two bounds $(\ref{eq:looser-bound})$ and $(\ref{eq:tight-bound})$ as
\banum
\|\hat{f} - f^o\|_{L_2(P_X)}^2 &= O_p\left(L\sum_{\ell=2}^L \bar{R}^{L-\ell+1}\lambda_\ell m_{\ell + 1}
+ \sum_{\ell=1}^L\frac{m_{\ell}m_{\ell+1}}{n}\log(n)\right)\label{eq:looser-order}
\eanum
and
\banum
\|\hat{f} - f^o\|_{L_2(P_X)}^2 &= O_p\left(L\sum_{\ell=2}^L \bar{R}^{L-\ell+1}\lambda_\ell 
+ \sum_{\ell=1}^L\frac{m_{\ell}m_{\ell+1}}{n}\log(n)\right),\label{eq:tight-order}
\eanum
respectively.
As mentioned in \cite{Suzuki18},
by balancing this bias-variance trade-off, 
we can determine the optimal widths of the internal layers.
Here, we will ignore the $\log(n)$-factor and $L$ for simplicity.
For the looser bound $(\ref{eq:looser-order})$, in order to balance the bias and variance terms, 
we set $\lambda_\ell$ as follows:
\[
\sum_{\ell=2}^L\lambda_\ell m_{\ell + 1} = \sum_{\ell=1}^L\frac{m_{\ell}m_{\ell+1}}{n}\]
Thus, we may set $m_\ell$ as follows:
\[
\lambda_\ell = \frac{m_\ell}{n}\;\text{ for }\; \ell=2,\dots,L.
\]
By contrast, 
for the tight bound $(\ref{eq:tight-order})$,
we set $\lambda_\ell$ as follows:
\[
\sum_{\ell=2}^L\lambda_\ell = \sum_{\ell=1}^L\frac{m_{\ell}m_{\ell+1}}{n}\le \sum_{\ell=1}^{L+1}\frac{m_{\ell}^2}{n},
\]
and thus we may set $m_\ell$ as follows:
\[
\lambda_\ell = \frac{m_\ell^2}{n}\;\text{ for }\; \ell=2,\dots,L.
\]
Combining these with the constraint $m_\ell \gtrsim N_{\ell}(\lambda_\ell)\log(N_{\ell}(\lambda_\ell))$ of $(\ref{eq:constraint})$,
the optimal widths $m_\ell$ in the internal layers, which minimizes the upper bound of the generalization error, 
can be determined.
Note that the optimal choice of $(m_\ell)_{\ell=2}^{L}$ based on $(\ref{eq:looser-bound})$ are different from
that based on our tight bound $(\ref{eq:tight-bound})$.

Now, we compare the looser and our new tight bounds 
under the corresponding best choices of $(m_\ell)_{\ell=2}^{L}$, respectively.
Here, we consider the setting in which 
the eigenvalue $\mu_j^{(\ell)}$ of $T_\ell$ decreases polynomially in $j$, that is, 
there exists constants $a_\ell>0$ and $s_\ell \in (0,1)$ such that
\banum
\mu_{j}^{(\ell)} \le a_\ell j^{-1/s_\ell}\quad(j\ge 1).\label{eq:eigen}
\eanum
This setting is commonly used in the analysis of kernel methods such as the support vector machine
(see, e.g., Section 7.7 in \cite{SteinwartChristmann08}).
Through the discussion in Section 4.4.2 in \cite{Suzuki17}, 
the degree of freedom $N_\ell(\lambda_\ell)$ can be evaluated as 
\[
N_{\ell}(\lambda_\ell) \lesssim(\lambda_\ell/a_\ell)^{-s_\ell}.
\]
In the looser bound $(\ref{eq:looser-order})$,
the optimal choice
\[
\lambda_\ell = a_{\ell}^{\frac{s_\ell}{1+s_\ell}}n^{-\frac{1}{1+s_\ell}}
\]
gives the looser generalization error bound:
\banum
\|\hat{f} - f^o\|_{L_2(P_X)}^2 \lesssim 
L\sum_{\ell=2}^L(\bar{R}\vee 1)^{2(L-\ell+1)}
a_{\ell}^{\frac{s_\ell}{1+s_\ell}}a_{\ell+1}^{\frac{s_{\ell+1}}{1+s_{\ell+1}}}
n^{\frac{s_\ell}{1+s_\ell}+\frac{s_{\ell+1}}{1+s_{\ell+1}}-1}
\log(n) + \frac{d_x^2}{n}\log(n),\label{eq:optimal-loose-bound}
\eanum
where the factors depending on $s_\ell, \log(\bar{R}R_b\hat{G}),\;\sigma^2$, and $\hat{R}_\infty$ are ignored.
By contrast, for our tight bound $(\ref{eq:tight-order})$, 
the optimal choice
\[
\lambda_\ell = a_{\ell}^{\frac{2s_\ell}{1+2s_\ell}}n^{-\frac{1}{1+2s_\ell}}
\]
provides
\banum
\|\hat{f} - f^o\|_{L_2(P_X)}^2 \lesssim 
L\sum_{\ell=2}^L(\bar{R}\vee 1)^{2(L-\ell+1)}
a_{\ell}^{\frac{2s_\ell}{1+2s_\ell}}n^{-\frac{1}{1+2s_\ell}}
\log(n) + \frac{d_x^2}{n}\log(n),\label{eq:optimal-tight-bound}
\eanum
which is the same bound in \cite{Suzuki18}.
To see the clear difference between $(\ref{eq:optimal-loose-bound})$ and $(\ref{eq:optimal-tight-bound})$, 
we simply assume that $s=s_2 = \cdots = s_{L}$.
Then, the looser and our tight generalization error are represented by
\[
\|\hat{f} - f^o\|_{L_2(P_X)}^2 = O_p(n^{-\frac{1-s}{1+s}})
\;\text{ and }\;
\|\hat{f} - f^o\|_{L_2(P_X)}^2 = O_p(n^{-\frac{1}{1+2s}}),
\]
respectively. 
If $s=1/3$, then
the looser generalization error bound $(\ref{eq:looser-bound})$ gives
\[
\|\hat{f} - f^o\|_{L_2(P_X)}^2 = O_p(n^{-1/2}).
\]
By contrast, in the same setting, our tight bound derived without scale invariance gives
\[
\|\hat{f} - f^o\|_{L_2(P_X)}^2 = O_p(n^{-3/5}).
\]
Moreover, if $s=1/2$, the looser bound $(\ref{eq:looser-bound})$ and our tight bound $(\ref{eq:tight-bound})$ lead a cubic-root rate and a square-root rate, 
respectively.
%

%
\section{Conclusion}

In this paper, using the framework developed in \cite{Suzuki18},
without the scale invariance of the activation functions, 
we derive the fast generalization error bound of deep learning.
This bound is essentially the same as that of \cite{Suzuki18}
although we {\it do not} assume the scale invariance of the activation functions.
Whereas we only focus on the empirical risk minimizer in this paper, 
we can also derive the tight generalization error bound of the Bayes estimator, under no assumption of the scale invariance,
by using the same derivation of \cite{Suzuki18} combining with our results (Theorem~\ref{theorem:approx-error} and Proposition~\ref{prop:covering-number}).
From the looser approximation error bound derived without the scale invariance in \cite{Suzuki17},
there is a possibility that 
a deep learning with the non scale invariant activations may have a slower convergence rate of $O(1/\sqrt{n})$
even when one with the scale invariant activations can reach a convergence rate faster than $O(1/\sqrt{n})$.
However, our tight analysis without using the scale invariance denies this possibility.
Hence, at least in the theoretical framework of \cite{Suzuki18}, 
we may conclude that the scale invariance of the activation functions is not essential to get the fast rate of convergence, 
and also that the theoretical framework proposed by \cite{Suzuki18} can be widely applied 
for analysis of deep learning with general activation functions.
In recent years, the non scale invariant activation functions including ELU (\citealp{ClevertEtAl16}) are proposed, 
and such activation functions empirically provide higher or comparable performance compared to ReLU.
Our analysis can be applied for the deep neural network with these activations.
Therefore, we believe our results contribute to the theoretical understanding of deep learning.
%


\bibliography{example_paper}

\appendix

\vspace{10pt}
Here, we provide the proofs of the results described above.
The overall flow of the proofs is the same as that of \cite{Suzuki18}.
On the other hand, since we do not assume the scale invariance of the activation functions,
the detail parts of the proofs of Theorem~\ref{theorem:approx-error}, Lemma~\ref{lemma:bound}, 
Lemma~\ref{lemma:difference-bound}, Proposition~\ref{prop:covering-number} are different.
From these results, the approach to derive the tight upper bound of the generalization error is the same as that of \cite{Suzuki18}.
For the sake of completeness and self-containedness, 
we provide the detailed proof of Theorem~\ref{theorem:main}.

%
\section{Proof of Theorem~\ref{theorem:approx-error}}\label{sec:proof-theorem:approx-error}

As with \cite{Suzuki18}, 
we construct the finite dimensional neural network approximating the true function $f^o$ recursively.
We follow the notation of \cite{Suzuki18}.
Here, $(v_j^{(\ell)})_{j=1}^{m_\ell}$ and $(w_j^{(\ell)})_{j=1}^{m_\ell}$ denote
the sequences in Proposition~\ref{prop:1}, and let $\hat{\Tcal}_\ell = \{v_{j}^{(\ell)}\}_{j=1}^{m_\ell}$.
By abuse of notation, we use the following notation:
\begin{itemize}
\item We use the same symbol $f_{\ell}^\ast$ for $f_{\ell}^\ast:\hat{\Tcal}_\ell \rightarrow \hat{\Tcal}_{\ell+1}$ and $f_\ell^\ast:\Rb^{m_{\ell}}\rightarrow \Rb^{m_\ell}$.
\item For function $F:\Rb^{d_x}\times \hat{\Tcal}_\ell \rightarrow \Rb$, we will denote by $f_{\ell}^\ast[F](x,v_i^{(\ell+1)})$ the following function.
\[
f_\ell^\ast[F(x,\cdot)](v_{i}^{(\ell+1)})=\sum_{j=1}^{m_\ell}W_{ij}^{\ell}F(x,v_{j}^{(\ell)})+b_i^{(\ell)} \text{ for }
v_{i}^{(\ell+1)}\in \hat{\Tcal}_{\ell+1}.
\]
\item When we denote $f_\ell^\ast[F]$ for $F:\Rb^{d_x}\times \Tcal_\ell\rightarrow \Rb\;((x,v)\mapsto F(x,v))$, 
$F$ will be regarded as its restriction on $\Rb^{d_x}\times \hat{\Tcal}_\ell$.
\item For $v \in \hat{\Tcal}_{\ell+1}$ and $x\in \Rb^{d_x}$, we define the output from the $\ell$-th layer of the approximator $f^\ast$ as $F_\ell^\ast(x,v)$. That is, the output is recursively given by $F_\ell^\ast(x,v)=f_\ell^\ast[F_{\ell-1}^\ast](x,v)$.
\item Similarly, we will an analogous notation for the true model $f_\ell^o$.
Write $F_\ell^o(x,v)=(f_\ell^o\circ\dots \circ f_1^o(x))(v)$ for $v\in \Tcal_{\ell+1}$ and $x\in \Rb^{d_x}$, and 
$F_\ell^o(x,v)=f_\ell^o[F_{(\ell-1)}^o](x,v)$.
\end{itemize}

\noindent{\bf Step 1} (the last layer, $\ell=L$):\;
First, we consider the approximation of the $L$-th layer.
Note that the output of the $L$-th layer is a single value.
Let $\Tcal_{L+1}=\{1\}$.
As the candidate of the approximation of the true $L$-th layer, 
we define the following approximator:
\banum\label{eq:S1}
\tilde{f}_L^\ast[F_{L-1}](x,1) 
= \sum_{j=1}^{m_L} \beta_{j}^{(L)}w_j^{(L)}\eta\left(F_{L-1}(x,v_j^{(L)})\right) + b_L.
\eanum
Here, according to Proposition~\ref{prop:1}, 
$\beta^{(L)}\in \Rb^{m_L}$ and $w^{(L)}\in \Rb^{m_L}$ satisfy $\|\beta^{(L)}\|_2^2\le c_1R^2/m_L$ and $\|w^{(L)}\|_2^2\le m_Lc_\delta$, 
respectively.
We set 
\[
W^{(L)}=W_{1,:}^{(L)}=({\beta^{(L)}}\odot w^{(L)})^T,\;b^{(L)}=b_L^o(1).
\]
Then, the model (\ref{eq:S1}) can be rewritten by
\[
\tilde{f}_L^\ast[F_{L-1}](x,1) =
\sum_{j=1}^{m_L} W_{1,j}^{(L)}\eta\left(F_{L-1}(x,v_j^{(L)})\right) + b_1^{(L)}.
\]
Note that the norms of $W^{(L)}$ and $b^{(L)}$ are bounded by
\banum
\|W_{1,:}^{(L)}\|_1 = \sum_{j=1}^{m_L}| \beta_{j}^{(L)}w_j^{(L)}|
\le \|\beta^{(L)}\|_2 \|w^{(L)}\|_2 \le \sqrt{c_1c_\delta}R,\quad
\|b^{(L)}\|_2 = |b_L| \le R_b,\label{eq:S2}
\eanum
respectively.
Thus, by Cauchy-Schwarz inequality and Assumption~\ref{assmp:1}, 
for any functions $F_{L-1}, F_{L-1}^\prime$ from $\hat{\Tcal}_L \times \Rb^{d_x}$ to $\Rb$,
\ba
&\int \left|\tilde{f}_L^\ast[F_{L-1}](x,1) -\tilde{f}_L^\ast[F_{L-1}^\prime](x,1)  \right|^2 P_X(dx)\\
&= 
\int \left| 
\sum_{j=1}^{m_L} \beta_{j}^{(L)}w_j^{(L)}\left\{
\eta\left(F_{L-1}(x,v_j^{(L)})\right) 
- \eta\left(F_{L-1}^\prime(x,v_j^{(L)})\right) 
\right\}
\right|^2 P_X(dx)\\
&\le
\left( \sum_{j=1}^{m_L}{\beta_{j}^{(L)}}^2\right)
 \left[ 
\sum_{j=1}^{m_L}{w_j^{(L)}}^2 \int \left\{ 
\eta\left(F_{L-1}(x,v_j^{(L)})\right) 
- \eta\left(F_{L-1}^\prime(x,v_j^{(L)})\right) 
\right\}^2P_X(dx)
\right]\\
&\le
\|\beta^{(L)}\|_2^2\|w^{(L)}\|_2^2
\left\|\left\{
\int\left|
F_{L-1}(x,v_j^{(L)}) - F_{L-1}^\prime(x,v_j^{(L)})
\right|^2P_X(dx)
\right\}_{j=1}^{m_L}
\right\|_{\max}\\
&\le 
c_1c_\delta R^2
\left\|\left\{
\int\left|
F_{L-1}(x,v_j^{(L)}) - F_{L-1}^\prime(x,v_j^{(L)})
\right|^2P_X(dx)
\right\}_{j=1}^{m_L}
\right\|_{\max}.
\ea
According to Proposition~\ref{prop:1}, we can take $\beta^{(L)}$ and $w^{(L)}$ such that
\[
\|\tilde{f}_L^\ast[F_{L-1}^o](\cdot,1) - f_L^o[F_{L-1}^o](\cdot,1)\|_{L_2(P_X)}^2 \le c_1\lambda R^2.
\]
Hereafter, we fix $\beta^{(L)}$ and $w^{(L)}$ satisfying this inequality and the bound (\ref{eq:S2}).

\noindent{\bf Step 2} (the internal layers, $\ell=2,\dots,L-1$):\;
Next, for the $\ell$-th internal layer, we will consider the following approximator:
\[
\tilde{f}_\ell^{\ast}[g](v_{i}^{(\ell+1)}) 
= \sum_{j=1}^{m_\ell}\beta_{i,j}^{(\ell)}w_{j}^{(\ell)}\eta(g(v_j^{(\ell)})) + b_\ell^{o}(v_{i}^{(\ell+1)}),
\]
where $g$ is a function from $\hat{\Tcal}$ to $\mathbb{R}$, 
$\beta^{\ell} \in \mathbb{R}^{m_{\ell+1}\times m_\ell}$ and $w^{(\ell)}\in \Rb^{m_\ell}$ 
satisfy $\|\beta_{j,:}^{\ell}\|_2^2 \le c_1R^2/m_\ell$ and $\|w^{(\ell)}\|_2^2\le m_{\ell}c_\delta$, respectively.
We set
\ba
W_{ij}^{(\ell)}=\beta_{ij}^{(\ell)}w_i^{(\ell)},\quad b^{(\ell)} =\left(b_\ell^{o}(v_{1}^{(\ell+1)}),\dots,b_\ell^{o}(v_{m_{\ell+1}}^{(\ell+1)}) \right)^T.
\ea
Then, we have
\[
\|W_{i,:}^{(\ell)}\|_1=\|\beta_{i,:}^{(\ell)}\odot w^{(\ell)}\|_1\le \|\beta_{i,:}^{(\ell)}\|_2\|w^{(\ell)}\|_2\le \sqrt{c_1c_\delta}R,\quad
\|b^{(L)}\|_{\max}\le R_b.
\]
For any $v_i^{(\ell+1)}\in \hat{\Tcal}_{(\ell+1)}$, we obtain
\ba
&\int \left|\tilde{f}_\ell^\ast[F_{\ell-1}](x,v_{i}^{(\ell+1)}) -\tilde{f}_\ell^\ast[F_{\ell-1}^\prime](x,v_{i}^{(\ell+1)}) \right|^2 P_X(dx)\\
&=
\int \left|
\sum_{j=1}^{m_\ell}\beta_{i,j}^{(\ell)}w_{j}^{(\ell)}
\left\{ \eta(F_{\ell-1}(x,v_j^{(\ell)})) 
-
\eta(F_{\ell-1}^\prime(x,v_j^{(\ell)})) \right\}
\right|^2 P_X(dx)\\
&\le 
\left(\sum_{j=1}^{m_\ell}{\beta_{i,j}^{(\ell)}}^2\right)
\left[
\sum_{j=1}^{m_\ell} {w_{j}^{(\ell)}}^2 \int \left| \eta(F_{\ell-1}(x,v_j^{(\ell)})) -\eta(F_{\ell-1}^\prime(x,v_j^{(\ell)}))\right|^2P_X(dx)
\right]\\
&\le
\|\beta_{i,:}^{(\ell)}\|_2^2\|w^{(\ell)}\|_2^2
\left\| \left\{ \int \left|
	F_{\ell-1}(x,v_j^{(\ell)}) -F_{\ell-1}^\prime(x,v_j^{(\ell)})
\right|^2P_X(dx) \right\}_{j=1}^{m_\ell} \right\|_{\max}\\
&\le
c_1c_\delta R^2
\left\| \left\{ 
\int \left| F_{\ell-1}(x,v_j^{(\ell)}) -F_{\ell-1}^\prime(x,v_j^{(\ell)}) \right|^2P_X(dx) 
\right\}_{j=1}^{m_\ell} 
\right\|_{\max}.
\ea
According to Proposition~\ref{prop:1}, we can choose $\beta^{(\ell)}$ and $w^{(\ell)}$ satisfying
\[
\max_{j=1,\dots,m_{\ell+1}}
\left\| 
\tilde{f}_\ell^\ast[F_{\ell-1}^o](\cdot,v_j^{(\ell+1)}) - f_\ell^o[F_{\ell-1}^o](\cdot,v_j^{(\ell+1)})
\right\|_{L_2(P_X)}^2 \le c_0\lambda_\ell R^2.
\]

\noindent{\bf Step 3} (the first layer, $\ell=1$):\;
In the first layer, for $v_i^{(2)}\in \hat{\Tcal}_2$, we set
\[
\tilde{f}_1^\ast(x,v_i^{(2)}) = \sum_{j=1}^{d_x} h_1^o(v_i^{(2)},j)\Qcal_1(j)x_j + b_1^o(v_i^{(2)}).
\]
From the definition of $f^o$, we have $\tilde{f}_1^\ast(x,v_i^{(2)})=f^o(x,v_i^{(2)})$.
Let 
\[
W^{(1)} =  \left( h_1^o(v_i^{(2)},j)\Qcal_1(j) \right)_{m_2\times d_x} \in \Rb^{m_2\times d_X},\quad
b^{(1)} = \left(b_1^o(v_1^{(2)}), \dots ,b_1^o(v_{m_2}^{(2)}) \right)^T.
\]
Then, by H\"{o}lder's inequality, 
\[
\|W^{(1)}\|_\infty
=
\max_{i=1,\dots,m_2} \sum_{j=1}^{d_x} |h_1^o(v_i^{(2)},j)|\Qcal_1(j)
\le 
\max_{i=1,\dots,m_2} \sqrt{\sum_{j=1}^{d_x} |h_1^o(v_i^{(2)},j)|^2\Qcal_1(j)}
\le R.
\]

\noindent{\bf Step 4}:\;
Finally, we combine the above results. 
The above inequalities derived without the scale invariance 
are the same as that of \cite{Suzuki18} except for the magnitudes of $W^{(\ell)}$ and $b^{(\ell)}$.
Hence, this part is the same as the corresponding part in \cite{Suzuki18}, in which the scale invariance is not required.
For the sake of completeness, we provide the part of the proof in \cite{Suzuki18}.
By the subadditivity of the norm,
\ba
&\left\|
f_L^o \circ f_{L-1}^o \circ \dots \circ  f_{1}^o
-
\tilde{f}_{L}^\ast \circ \tilde{f}_{L-1}^\ast \circ \dots \circ  \tilde{f}_{1}^\ast
\right\|_{L_2(P_X)}\\
&=
\Bigl\|
(f_L^o \circ f_{L-1}^o \circ \dots \circ  f_{1}^o) - (\tilde{f}_{L}^\ast \circ f_{L-1}^o \circ \dots \circ  f_{1}^o)\\
&\qquad\;\;\vdots\\
&\qquad + 
(\tilde{f}_{L}^\ast \circ \dots \circ \tilde{f}_{\ell + 1}^\ast \circ f_{\ell}^o \circ f_{\ell-1}^o \circ \dots \circ  f_{1}^o)
-
(\tilde{f}_{L}^\ast \circ \dots \circ \tilde{f}_{\ell + 1}^\ast \circ \tilde{f}_{\ell}^\ast \circ f_{\ell-1}^o \circ \dots \circ  f_{1}^o)\\
&\qquad\;\;\vdots\\
&\qquad+
(\tilde{f}_{L}^\ast \circ \dots \circ \tilde{f}_{2}^\ast \circ f_{\ell}^o)
-
(\tilde{f}_{L}^\ast \circ \dots \circ \tilde{f}_{2}^\ast \circ \tilde{f}_{1}^\ast)\Bigr\|_{L_2(P_X)}\\
&\le 
\sum_{\ell=1}^L
\left\|
(\tilde{f}_{L}^\ast \circ \dots \circ \tilde{f}_{\ell + 1}^\ast \circ f_{\ell}^o \circ f_{\ell-1}^o \circ \dots \circ  f_{1}^o)
-
(\tilde{f}_{L}^\ast \circ \dots \circ \tilde{f}_{\ell + 1}^\ast \circ \tilde{f}_{\ell}^\ast \circ f_{\ell-1}^o \circ \dots \circ  f_{1}^o)
\right\|_{L_2(P_X)}.
\ea
Combining these, we obtain
\ba
&\left\|
(\tilde{f}_{L}^\ast \circ \dots \circ \tilde{f}_{\ell + 1}^\ast \circ f_{\ell}^o \circ f_{\ell-1}^o \circ \dots \circ  f_{1}^o)
-
(\tilde{f}_{L}^\ast \circ \dots \circ \tilde{f}_{\ell + 1}^\ast \circ \tilde{f}_{\ell}^\ast \circ f_{\ell-1}^o \circ \dots \circ  f_{1}^o)
\right\|_{L_2(P_X)}\\
&\le (\sqrt{c_1c_\delta}R)^{L-\ell}\sqrt{c_0\lambda_\ell}R
\le R^{L-\ell+1}\sqrt{(c_1c_\delta)^{L-\ell} c_0}\sqrt{\lambda_\ell}.
\ea
Therefore, we conclude that
\ba
\|f^o-f^\ast\|_{L_2(P_X)}\le \sum_{\ell=2}^L R^{L-\ell+1}\sqrt{(c_1c_\delta)^{L-\ell} c_0}\sqrt{\lambda_\ell},
\ea
and the proof is complete.
%

%
\section{Proof of Lemma~\ref{lemma:bound}}\label{sec:proof-lemma:bound}

First, we will derive the upper bound of $\|f^o\|_\infty$.
Suppose that $\|F_{\ell-1}^o(x,\cdot)\|_{L_2(\Qcal_\ell)}\le G$.
Then, for any $\tau \in \Tcal_{\ell+1}$, we have
\ba
|F_\ell^o(x,\tau)|
&=
\left| 
\int_{\Tcal_\ell} h_\ell^o(\tau, w) \eta(F_{\ell-1}^o(x,w))\,\drm \Qcal_\ell(w) + b_\ell^o(\tau)
\right|\\
&\le \|h_\ell^o(\tau,\cdot)\|_{L_2(\Qcal_\ell)}\|\eta(F_{\ell-1}^o(\tau,\cdot))\|_{L_2(\Qcal_\ell)} + |b_\ell^o(\tau)|\\
&\le \|h_\ell^o(\tau,\cdot)\|_{L_2(\Qcal_\ell)}\||\eta(F_{\ell-1}^o(\tau,\cdot)) - \eta(0)| + |\eta(0)|\|_{L_2(\Qcal_\ell)} + |b_\ell^o(\tau)|\\
&\le \|h_\ell^o(\tau,\cdot)\|_{L_2(\Qcal_\ell)}\left\{\|F_{\ell-1}^o(\tau,\cdot)\|_{L_2(\Qcal_\ell)} + c_\eta\right\} + |b_\ell^o(\tau)|\\
&\le R(G+c_\eta) + R_b.
\ea
By H\"{o}lder's inequality, for any $\tau\in \Tcal_2$ and any $x\in \Rb^{d_x}$,
\ba
|f_1^o(x,\tau)| 
&= \left| \sum_{i=1}^{d_x} h_1^o(\tau,i)x_i\Qcal_1(i)  + b_1^o(\tau)\right|
\le \sum_{i=1}^{d_x} |h_1^o(\tau,i)x_i\Qcal_1(i)|  + |b_1^o(\tau)|\\
&\le \|x\|_{\max} \sum_{i=1}^{d_x} |h_1^o(\tau,i)|\Qcal_1(i)  + |b_1^o(\tau)|
\le \|x\|_{\max}\|h_1^o(\tau,\cdot)\|_{L_2(\Qcal_1)}  + |b_1^o(\tau)|\\
&\le RD_x + R_b.
\ea
Combining these, we obtain
\ba
\|f^o\|_\infty
\le 
R^L D_x + \sum_{\ell = 1}^LR^{L-\ell} (R_b+c_\eta).
\ea

We next prove the upper bound of $\|f\|_\infty$.
Denote $a^{(\ell)}(x):=(a_1^{(\ell)}(x),\dots,a_{m_{\ell+1}}^{(\ell)}(x))^T:=(W^{(\ell)}\eta(\cdot) + b^{(\ell)}) \circ \dots \circ (W^{(1)}x + b^{(1)})\in \Rb^{m_{\ell+1}}$.
By Assumption~\ref{assmp:1}, we have
\[
|\eta(a_j^{(\ell)}(x)) - \eta(0) + \eta(0)|
\le |\eta(a_j^{(\ell)}(x)) - \eta(0)|+ c_\eta
\le |a_j^{(\ell)}(x)| + c_\eta.
\]
Thus, we obtain
\ba
\|f\|_\infty
&=\| (W^{(L)}\eta(\cdot) + b^{(L)})\circ \dots \circ (W^{(1)}x + b^{(1)})\|_\infty\\
&\le
\sup_{x\in \Xc} \left|
W^{(L)}\eta\left(
(W^{(L-1)}\eta(\cdot) + b^{(L-1)}) \circ \dots \circ (W^{(1)}x + b^{(1)})
\right) \right| + |b^{(L)}|\\
&= \sup_{x\in \Xc} \left|W^{(L)}\eta(a^{(L-1)}(x)) \right| + |b^{(L)}|
\le
\sup_{x\in \Xc} \sum_{j=1}^{m_L} |W_{1j}^{(L)}\eta(a_j^{(L-1)}(x))|  + |b^{(L)}|\\
&=
\|W^{(L)}\|_\infty\sup_{x\in \Xc}\left\| \{\eta(a_j^{(L-1)}(x))\}_{j=1}^{m_L}\right\|_{\max}  + |b^{(L)}|
\le \bar{R}\sup_{x\in \Xc}\left\| \{\eta(a_j^{(L-1)}(x))\}_{j=1}^{m_L}\right\|_{\max}  + R_b\\
&\le \bar{R}\left[ \sup_{x\in \Xc}\left\| \{a_j^{(L-1)}(x)\}_{j=1}^{m_L}\right\|_{\max} + c_\eta\right]  + R_b.
\ea
For $\ell = 2,\dots, L-1$, we have
\ba
\left\| \{a_j^{(\ell)}(x)\}_{j=1}^{m_{\ell+1}}\right\|_{\max}
&=
\left\| \left\{{W_{j,:}^{(\ell)}}^T\eta(a^{(\ell-1)}(x)) + b_j^{(\ell)} \right\}_{j=1}^{m_{\ell}}\right\|_{\max}\\
&\le
\left\|{W^{(\ell)}}\right\|_\infty\left\| \left\{|\eta(a^{(\ell-1)}(x))|\right\}_{j=1}^{m_{\ell}}\right\|_{\max} + |b_j^{(\ell)}|\\
&\le
\bar{R}\left[ \left\| \left\{|a^{(\ell-1)}(x)|\right\}_{j=1}^{m_{\ell}}\right\|_{\max} + c_\eta \right] + R_b.
\ea
Moreover, for $\ell = 1$, 
\[
\left\| \{a_j^{(1)}(x)\}_{j=1}^{m_{2}}\right\|_{\max}
=\left\| \{ {W_{j,:}^{(1)}}^Tx + b_j^{(1)}\}_{j=1}^{m_{2}}\right\|_{\max}
\le \|W^{(1)}\|_\infty\|x\|_{\max} + \left\| b^{(1)}\right\|_{\max}
\le \bar{R}D_x + R_b.
\]
Combining these, we conclude that
\ba
\|f\|_\infty 
&\le 
\bar{R}\sup_{x\in \Xc}\left\| \{a_j^{(L-1)}(x)\}_{j=1}^{m_L}\right\|_{\max} + \bar{R}c_\eta + R_b\\
&\le 
\bar{R}^2\sup_{x\in \Xc}\left\| \{a_j^{(L-2)}(x)\}_{j=1}^{m_L}\right\|_{\max} + \bar{R}^2c_\eta + \bar{R}c_\eta + \bar{R}R_b+R_b\\
&\le 
\bar{R}^{L-1}\sup_{x\in \Xc}\left\| \{a_j^{(1)}(x)\}_{j=1}^{m_L}\right\|_{\max} + c_\eta\sum_{\ell=1}^{L-1}\bar{R}^\ell  +R_b \sum_{\ell=0}^{L-2}\bar{R}^\ell\\
&\le 
\bar{R}^{L-1}\left\{ \bar{R}D_x + R_b \right\} + c_\eta\sum_{\ell=1}^{L-1}\bar{R}^\ell  +R_b \sum_{\ell=0}^{L-2}\bar{R}^\ell\\
&\le 
\bar{R}^{L}D_x+ c_\eta\sum_{\ell=1}^{L-1}\bar{R}^{L-\ell}  +R_b \sum_{\ell=1}^{L}\bar{R}^{L-\ell}
\le 
\bar{R}^{L}D_x + (c_\eta+R_b) \sum_{\ell=1}^{L}\bar{R}^{L-\ell},
\ea
which completes the proof.
%

%
\section{Proof of Lemma~\ref{lemma:difference-bound}}\label{sec:proof-lemma:difference-bound}

Let $f,f^\prime \in \Fcal$ be two functions with parameters $(W^{(\ell)},b^{(\ell)})_{\ell = 1}^L$ and $({W^\prime}^{(\ell)},{b^\prime}^{(\ell)})_{\ell = 1}^L$, respectively.
Assume that $\|W^{(\ell)} - {W^\prime}^{(\ell)}\|_\infty<\epsilon$ and $\|b^{(\ell)}-{b^\prime}^{(\ell)}\|_{\max}<\epsilon\;(\ell = 1,\dots,L)$.
Then, we have
\ba
&\|f-f^\prime\|_\infty\\
&=
\left\|
\left( W^{(L)}\eta(a^{(L-1)}(x)) + b^{(L)}\right) - \left({W^\prime}^{(L)}\eta({a^\prime}^{(L-1)}(x)) + {b^\prime}^{(L)}\right)
\right\|_\infty\\
&\le 
\left\|\left(W^{(L)} - {W^\prime}^{(L)}\right)\eta(a^{(L-1)}(x))\right\|_\infty
+
\left\|  {W^\prime}^{(L)}\left[\eta(a^{(L-1)}(x)) - \eta({a^\prime}^{(L-1)}(x))\right] \right\|_\infty
+
\left| b^{(L)} - {b^\prime}^{(L)} \right|\\
&\le 
\sup_{x\in \Xc}\left[ 
\left\| \left(W^{(L)} - {W^\prime}^{(L)}\right)\right\|_\infty \left\|\eta(a^{(L-1)}(x))\right\|_{\max}
+
\left\|  {W^\prime}^{(L)} \right\|_\infty\left\| \eta(a^{(L-1)}(x)) - \eta({a^\prime}^{(L-1)}(x)) \right\|_{\max}
\right]
+\epsilon\\
&\le \epsilon\sup_{x\in \Xc}\left\|\left\{ \eta(a^{(L-1)}(x)) \right\}_{j=1}^{m_L}\right\|_{\max}
+ \bar{R}\sup_{x\in \Xc}\left\|a^{(L-1)}(x) - {a^\prime}^{(L-1)}(x) \right\|_{\max}+\epsilon\\
&\le \epsilon \left\{ \bar{R}^{L-1}D_x +  c_\eta\sum_{\ell = 2}^{L-1}\bar{R}^{L-\ell} + R_b\sum_{\ell = 2}^{L}\bar{R}^{L-\ell} \right\}
+ \bar{R}\sup_{x\in \Xc}\left\|a^{(L-1)}(x) - {a^\prime}^{(L-1)}(x) \right\|_{\max}+\epsilon.
\ea
For $\ell = 2,\dots,L-1$, it follows that
\ba
&\left\|a^{(\ell)}(x) - {a^\prime}^{(\ell)}(x) \right\|_{\max}\\
&=
\left\|
\left\{
\left( W_{j,:}^{(\ell)}\eta(a^{(\ell-1)}(x)) + b_j^{(\ell)} \right) -  \left(  {W^\prime}_{j,:}^{(\ell)}\eta({a^\prime}^{(\ell-1)}(x)) + {b_j^\prime}^{(\ell)} \right)
\right\}_{j=1}^{m_{\ell+1}}
\right\|_{\max}\\
&=
\left\|
\left\{
\left( W_{j,:}^{(\ell)}- {W^\prime}_{j,:}^{(\ell)}\right)\eta(a^{(\ell-1)}(x)) 
+ {W^\prime}_{j,:}^{(\ell)}\left[\eta(a^{(\ell-1)}(x)) -  \eta({a^\prime}^{(\ell-1)}(x)) \right]+ \left(b_j^{(\ell)} -  {b_j^\prime}^{(\ell)} \right)
\right\}_{j=1}^{m_{\ell+1}}
\right\|_{\max}\\
&\le 
\left\|
\left\{
\left( W_{j,:}^{(\ell)}- {W^\prime}_{j,:}^{(\ell)}\right)\eta(a^{(\ell-1)}(x)) 
\right\}_{j=1}^{m_{\ell+1}}
\right\|_{\max}
+
\left\|
\left\{
{W^\prime}_{j,:}^{(\ell)}\left[\eta(a^{(\ell-1)}(x)) -  \eta({a^\prime}^{(\ell-1)}(x)) \right]
\right\}_{j=1}^{m_{\ell+1}}
\right\|_{\max}
+\epsilon\\
&\le
\left\|
\left\{
\left\|W_{j,:}^{(\ell)}- {W^\prime}_{j,:}^{(\ell)}\right\|_1\left\|\eta(a^{(\ell-1)}(x)) \right\|_{\max}
\right\}_{j=1}^{m_{\ell+1}}
\right\|_{\max}\\
&\qquad\qquad\qquad+
\left\|
\left\{
\left\|{W^\prime}_{j,:}^{(\ell)}\right\|_1\left\|\eta(a^{(\ell-1)}(x)) -  \eta({a^\prime}^{(\ell-1)}(x)) \right\|_{\max}
\right\}_{j=1}^{m_{\ell+1}}
\right\|_{\max}+\epsilon\\
&\le
\left\|W^{(\ell)}- {W^\prime}^{(\ell)}\right\|_\infty\left\|\eta(a^{(\ell-1)}(x)) \right\|_{\max}
+
\left\|{W^\prime}^{(\ell)}\right\|_\infty\left\|\eta(a^{(\ell-1)}(x)) -  \eta({a^\prime}^{(\ell-1)}(x)) \right\|_{\max}
+\epsilon\\
&\le 
\epsilon \left\|\eta(a^{(\ell-1)}(x)) \right\|_{\max} + \bar{R}\left\|a^{(\ell-1)}(x) -  {a^\prime}^{(\ell-1)}(x) \right\|_{\max}
+\epsilon.
\ea
Since 
\[
\left\|\eta(a^{(\ell-1)}(x)) \right\|_{\max} \le \bar{R}^{\ell} D_x + (c_\eta + R_b)\sum_{m = L-\ell + 1}^{L} \bar{R}^{L-m},
\]
we obtain
\ba
\left\|a^{(\ell)}(x) - {a^\prime}^{(\ell)}(x) \right\|_{\max}
\le 
\epsilon\left\{ \bar{R}^{\ell} D_x + (c_\eta + R_b)\sum_{m = L-\ell + 1}^{L} \bar{R}^{L-m} \right\} 
+ \bar{R}\left\|a^{(\ell-1)}(x) -  {a^\prime}^{(\ell-1)}(x) \right\|_{\max}
+\epsilon.
\ea
Furthermore, it follows that
\ba
\left\|a^{(1)}(x) - {a^\prime}^{(1)}(x) \right\|_{\max}
&\le 
\left\|
\left\{
\left( W_{j,:}^{(1)}x + b_j^{(1)} \right)- \left( {W^\prime}_{j,:}^{(1)}x + {b^\prime}_j^{(1)} \right) 
\right\}_{j=1}^{m_2}
\right\|_{\max}\\
&\le
\left\|
W^{(1)} - {W^\prime}^{(1)}
\right\|_\infty\|x\|_{\max} + \epsilon
\le 
\epsilon D_x + \epsilon.
\ea
These inequalities yield
\ba
&\|f-f^\prime\|_\infty\\
&\le \epsilon \left\{ \bar{R}^{L-1}D_x +  (c_\eta+R_b)\sum_{\ell = 2}^{L}\bar{R}^{L-\ell} \right\}
+ \bar{R}\sup_{x\in \Xc}\left\|a^{(L-1)}(x) - {a^\prime}^{(L-1)}(x) \right\|_{\max}+\epsilon\\
&\le 
\epsilon \left\{ \bar{R}^{L-1}D_x +  (c_\eta+R_b)\sum_{\ell = 2}^{L}\bar{R}^{L-\ell} \right\}\\
&\quad+
\bar{R}
\left\{
\epsilon\left[ \bar{R}^{L-2} D_x + (c_\eta + R_b)\sum_{m =  3}^{L} \bar{R}^{L-m} \right]
+ \sup_{x\in \Xc}\bar{R}\left\|a^{(L-2)}(x) -  {a^\prime}^{(L-2)}(x) \right\|_{\max}
+\epsilon
\right\}
+\epsilon\\
&\le 
\epsilon 
\left\{
\left[ \bar{R}^{L-1}D_x +  (c_\eta+R_b)\sum_{\ell = 2}^{L}\bar{R}^{L-\ell} \right]
+\bar{R}\left[ \bar{R}^{L-2} D_x + (c_\eta + R_b)\sum_{m =  3}^{L} \bar{R}^{L-m} \right]
\right\}+\left\{ \epsilon + \bar{R}\epsilon\right\}\\
&\quad + \sup_{x\in \Xc}\bar{R}^2\left\|a^{(L-2)}(x) -  {a^\prime}^{(L-2)}(x) \right\|_{\max}\\
&\le 
\epsilon 
\left\{
\sum_{\ell = 2}^{L}\bar{R}^{L-\ell}\left[ \bar{R}^{\ell-1}D_x +  (c_\eta+R_b)\sum_{m = L- \ell + 1}^{L}\bar{R}^{L-m} \right]
\right\} + \epsilon \sum_{\ell = 2}^{L}\bar{R}^{L-\ell}\\
&\quad+ 
\bar{R}^{L-1}\sup_{x\in \Xc}\left\|a^{(1)}(x) -  {a^\prime}^{(1)}(x) \right\|_{\max}\\
&\le
\epsilon 
\left\{
\sum_{\ell = 2}^{L}\bar{R}^{L-\ell}\left[ \bar{R}^{\ell-1}D_x 
+  (c_\eta+R_b)\sum_{m= 0}^{\ell-1}\bar{R}^{m} \right]
\right\} + \epsilon \sum_{\ell = 2}^{L}\bar{R}^{L-\ell}
+
\bar{R}^{L-1}\left\{ \epsilon D_x + \epsilon \right\}\\
&\le
\epsilon 
\left\{
\sum_{\ell = 1}^{L}\bar{R}^{L-1}D_x +  (c_\eta+R_b)
\sum_{\ell = 2}^{L}\sum_{m= 0}^{\ell-1}\bar{R}^{L-\ell+m} 
\right\} + \epsilon \sum_{\ell = 1}^{L}\bar{R}^{L-\ell}\\
&\le
\epsilon 
\left\{
L\bar{R}^{L-1}\left[ D_x +  L(c_\eta+R_b) \right]
+  \sum_{\ell = 1}^{L}\bar{R}^{L-\ell}
\right\},
\ea
and the lemma follows.
%

%
\section{Proof of Proposition~\ref{prop:covering-number}}\label{sec:proof-prop:covering-number}

Let
\[
\hat{G} = L\bar{R}^{L-1}\left[ D_x +  L(c_\eta+R_b) \right]
+  \sum_{\ell = 1}^{L}\bar{R}^{L-\ell}.
\]
Let $f,f^\prime \in \Fcal$ be two functions with parameters $(W^{(\ell)},b^{(\ell)})_{\ell = 1}^L$ and $({W^\prime}^{(\ell)},{b^\prime}^{(\ell)})_{\ell = 1}^L$, respectively.
From Lemma~\ref{lemma:difference-bound}, 
we have $\|f-f^\prime\|_\infty \le \delta$
if $\|W^{(\ell)} - {W^\prime}^{(\ell)}\|_\infty<\delta/\hat{G},\;\|b^{(\ell)}-{b^\prime}^{(\ell)}\|_{\max}<\delta/\hat{G}\;(\ell = 1,\dots,L)$.

In much the same way as Section 4.2 in \cite{Vershynin18}, 
we will derive the following upper bound for the covering number $N(\epsilon, B_p^d, \|\cdot\|_p)$ of the unit $d$-dimensional $L_p$-ball $B_p^d\subset \Rb^d$ under the $L_p$ distance:
\banum
N(\epsilon, B_p^d, \|\cdot\|_p) \le \left(1 + \frac{2}{\epsilon}\right)^d.\label{eq:cn-ball}
\eanum
For $A\subset \Rb^d$, let $\mathrm{Vol}(A)$ denote the volume of $A$.
Here, we note that, for any $a>0$, $\mathrm{Vol}(aA) = a^d\mathrm{Vol}(A)$, where $aA=\{ax \mid x\in A\}$.
We will denote by $\{c_1,\dots,c_m\}$ an $\epsilon$-packing. That is, $\{c_1,\dots,c_m\}$ satisfies $\min_{i\neq j}\|c_i-c_j\|\ge \epsilon$.
Denote $B_p^d(c, \delta):=\{x\in \Rb^d \mid \|c - x \|_p\le \delta  \}$.
Then, the interiors of balls $B_p^d(c_i, \epsilon/2)\;(i = 1,\dots,m)$ are disjoint, and 
\[
B_p^d(c_i, \epsilon/2) \subset B_p^d + \frac{\epsilon}{2}B_p^d\quad(i = 1,\dots,m),
\]
where $A+B$ denotes the Minkowski sum of two sets $A$ and $B$.
Hence, 
\[
\mathrm{Vol}\left(B_p^d + \frac{\epsilon}{2}B_p^d \right) \ge m \left(\frac{\epsilon}{2}\right)^d\mathrm{Vol}( B_p^d).
\]
It follows that the $\epsilon$-packing number $M(\epsilon, B_p^d, \|\cdot\|_p)$, that is, the largest possible cardinality of an $\epsilon$-packing of $B_p^d$, 
can be bounded by
\[
M(\epsilon, B_p^d, \|\cdot\|_p) \le \left(\frac{2}{\epsilon}\right)^d\frac{\mathrm{Vol}\left(B_p^d + \frac{\epsilon}{2}B_p^d \right)}{\mathrm{Vol}( B_p^d)},
\]
which gives $(\ref{eq:cn-ball})$ combined with Lemma 4.2.8 of \cite{Vershynin18}.
Accordingly, the $\epsilon$-covering number of $\Fcal$ is bounded as follows:
\ba
&\log N(\epsilon, \Fcal, \|\cdot\|_\infty)\\
&\le 
\sum_{\ell = 1}^L \log N( \epsilon/\hat{G}, B_\infty^{m_{\ell+1}\times m_\ell}(\bar{R}), \|\cdot\|_\infty)
+
\sum_{\ell = 1}^L \log N( \epsilon/\hat{G}, B_{\max}^{m_{\ell}}(R_b), \|\cdot\|_{\max})\\
&\le 
\sum_{\ell = 1}^L \log \prod_{i= 1}^{m_{\ell+1}}N( \epsilon/(\hat{G}\bar{R}), B_1^{m_\ell}(1), \|\cdot\|_1)
+
\sum_{\ell = 1}^L \log N( \epsilon/(\hat{G}R_b), B_{\max}^{m_{\ell}}(1), \|\cdot\|_{\max})\\
&\le
\sum_{\ell = 1}^L \log \prod_{i= 1}^{m_{\ell+1}} \left\{ 1 + \frac{2}{\epsilon/(\hat{G}\bar{R})} \right\}^{m_\ell}
+
\sum_{\ell = 1}^L \log \left\{ 1 + \frac{2}{\epsilon/(\hat{G}R_b)}\right\}^{m_\ell}\\
&\le 
\log\left\{ 1 + \frac{2\hat{G}\max\{\bar{R},R_b\}}{\epsilon} \right\}
\left\{ \sum_{\ell = 1}^L (m_{\ell+1}+1) m_\ell \right\}.
\ea
%

%
\section{Proof of Theorem~\ref{theorem:main}}\label{sec:proof-theorem:main}

We remark that the approach to derive the tight upper bound of the generalization error from the above results is the same as that of \cite{Suzuki18}.
Nevertheless, here we provide the detailed proof of Theorem~\ref{theorem:main} just for the sake of completeness and self-containedness.
We can clearly see that the scale invariance of the activation functions is not required in the following derivation.

For fixed input $x_1,\dots,x_n\in \Rb^{d_X}$ and for $f\in \Fcal$, 
define $\|f\|_n^2 := \sum_{i=1}^n f(x_i)^2/n$.

\subsection{Evaluation of $\|\hat{f}-f^\ast\|_n$}
Let 
\[
\Gcal_\delta := \{f-f^\ast \mid \|f-f^\ast\|_n\le \delta , f\in \Fcal\}.
\]
For $g\in \Gcal_{2\delta} $, we will consider the following process $X$:
\[
X: \Gcal_{2\delta} \rightarrow \Rb,\quad g \mapsto \frac{1}{\sqrt{n}}\sum_{i=1}^n \frac{\xi_i}{\sigma}g(x_i).
\]
Here, note that $X$ is a sub-Gaussian process.
According to Theorem 2.5.8 in \cite{GineNickl15}, we have
 \banum
 \Pb\left( 
 \sup_{g \in \Gcal_{2\delta}}\left| \frac{1}{n}\sum_{i=1}^n\xi_i g(x_i) \right| 
 \ge \Eb \left[ \sup_{g \in \Gcal_{2\delta}}\left| \frac{1}{n}\sum_{i=1}^n\xi_i g(x_i) \right| \right] 
 + 2\delta r\right) 
 \le \exp\left\{ -\frac{nr^2}{2\sigma^2} \right\}.\label{eq:theorem-2.5.8}
\eanum
Note that $f\in\Fcal$ implies $f-f^\ast\in\Gcal_{2\hat{R}_\infty}$.
Applying the inequality $(\ref{eq:theorem-2.5.8})$ for $\delta_j=2^{j-1}\sigma/\sqrt{n}\;(j=1,\dots,\lceil\log_2(\hat{R}_\infty\sqrt{n}/\sigma) \rceil + 1)$ repeatedly, 
we see that, for any $\delta\ge \sigma/\sqrt{n}$,
\ba
&\Pb\left( 
\bigcup_{j = 1}^{\lceil\log_2(\hat{R}_\infty\sqrt{n}/\sigma) \rceil + 1}\left\{ 
 \sup_{g\in \Gcal_{\delta_j}}\left| \frac{1}{n}\sum_{i=1}^n\xi_i g(x_i) \right| 
 \ge \Eb \left[ \sup_{g \in \Gcal_{\delta_j}}\left| \frac{1}{n}\sum_{i=1}^n\xi_i g(x_i) \right| \right] 
 + \delta_j r\right\}
 \right)\\
 &\le 
 (\lceil\log_2(\hat{R}_\infty\sqrt{n}/\sigma) \rceil+1)\times \exp\left\{ -\frac{nr^2}{2\sigma^2} \right\}.
\ea

Now, we will evaluate the expectation $\Eb \left[ \sup_{g \in \Gcal_{2\delta}}\left| \sum_{i=1}^n\xi_i g(x_i)/n \right| \right]$.
We remark that, for any constant $B>0$, a simple computation gives 
\ba
\int_{0}^{2\delta}\sqrt{\log\left(1 + \frac{B}{\epsilon} \right)} \,\drm \epsilon
&\le\int_{0}^{2\delta}\sqrt{\log\left(\frac{2\delta+B}{\epsilon} \right)} \,\drm \epsilon
\le
2\delta \left\{ \sqrt{\log\left(1+\frac{B}{2\delta}\right)}
+\sqrt{\pi}\right\}.
\ea
Since $\|f\|_n\le \|f\|_\infty$, we have $\log(2N(\epsilon,\mathcal{G}_{2\delta},\|\cdot\|_n))\le \log(2N(\epsilon,\Fcal,\|\cdot\|_\infty))$.
From Proposition~\ref{prop:covering-number}, it follows that
\ba
\int_{0}^{2\delta} \sqrt{\log(2N(\epsilon,\mathcal{G}_{2\delta},\|\cdot\|_n))}\,\drm \epsilon
&\le
\int_{0}^{2\delta} \sqrt{\log(2N(\epsilon,\Fcal,\|\cdot\|_\infty))}\,\drm \epsilon\\
&\le 
C \delta\sqrt{\left\{\sum_{\ell = 1}^L (m_{\ell+1}+1) m_\ell\right\}\log_+\left(1+\frac{\hat{G}\max\{\bar{R},R_b\}}{\delta}\right)},
\ea
where $C$ is a universal constant.
By Theorem 2.3.6 in \cite{GineNickl15}, we obtain
\ba
\Eb\left[\left| \sup_{g\in \Gcal_{2\delta}} \frac{1}{n}\sum_{i=1}^n\xi_i g(x_i) \right|  \right]
&=
\frac{\sigma}{\sqrt{n}}\Eb\left[\left| \sup_{g\in \Gcal_{2\delta}} \frac{1}{\sqrt{n}}\sum_{i=1}^n\frac{\xi_i}{\sigma} g(x_i) \right|  \right]\\
&\le
4\sqrt{2}\frac{\sigma}{\sqrt{n}}\int_{0}^{2\delta} \sqrt{\log(2N(\epsilon,\mathcal{G}_{2\delta},\|\cdot\|_n))}\,\drm \epsilon\\
&\le 
4\sqrt{2}\frac{\sigma}{\sqrt{n}}C \delta\sqrt{\left\{\sum_{\ell = 1}^L (m_{\ell+1}+1) m_\ell\right\}\log_+\left(1+\frac{\hat{G}\max\{\bar{R},R_b\}}{\delta}\right)}\\
&\le 
C^\prime \sigma \delta\sqrt{\frac{\sum_{\ell = 1}^L m_{\ell+1} m_\ell}{n}\log_+\left(1+\frac{\hat{G}\max\{\bar{R},R_b\}}{\delta}\right)},
\ea
where $C$ and $C^\prime$ are universal constants.

In these inequalities, we take $\delta$ and $r$ as
\[
\delta \leftarrow \left( \|f-f^\ast\|_n \vee \sigma\sqrt{\frac{\sum_{\ell = 1}^L m_{\ell+1} m_\ell}{n}}\right) \;\text{ and }\;
r \leftarrow \sigma r/\sqrt{n},
\]
respectively.
Then, by $ab\le a^2/4 + b^2$ and $(a+b)^2 \le 2(a^2 + b^2)$, 
with probability at least $1-(\lceil\log_2(\hat{R}_\infty\sqrt{n}/\sigma) \rceil+1) \exp(-r^2/2)$, 
we have that, uniformly for all $f\in \Fcal$, 
\banum
&\left|\frac{1}{n}\sum_{i=1}^n\xi_i \{f(x_i) - f^\ast(x_i) \}\right|\nonumber\\
&\le 
C \sigma \left( \|f-f^\ast\|_n \vee \sigma\sqrt{\frac{\sum_{\ell = 1}^L m_{\ell+1} m_\ell}{n}}\right)
\sqrt{\frac{\sum_{\ell = 1}^L m_{\ell+1} m_\ell}{n}
\log_+\left(1+\frac{\sqrt{n}\hat{G}\max\{\bar{R},R_b\}}{\sigma \sqrt{\sum_{\ell = 1}^L m_{\ell+1} m_\ell}}\right)}\nonumber\\
&\quad+2\left( \|f-f^\ast\|_n \vee \sigma\sqrt{\frac{\sum_{\ell = 1}^L m_{\ell+1} m_\ell}{n}}\right)  \sigma \frac{r}{\sqrt{n}}\nonumber\\
&\le
\frac{1}{4}\left( \|f-f^\ast\|_n \vee \sigma\sqrt{\frac{\sum_{\ell = 1}^L m_{\ell+1} m_\ell}{n}}\right)^2\nonumber\\
&\quad 
+2C^2\sigma^2\left\{\frac{\sum_{\ell = 1}^L m_{\ell+1} m_\ell}{n}
\log_+\left(1+\frac{\sqrt{n}\hat{G}\max\{\bar{R},R_b\}}{\sigma \sqrt{\sum_{\ell = 1}^L m_{\ell+1} m_\ell}}\right)
+ 4\frac{r^2}{n}\right\}.\label{eq:complex-inequality1}
\eanum
Let
\[
\Psi_{r,n}:=2C^2\sigma^2\left\{\frac{\sum_{\ell = 1}^L m_{\ell+1} m_\ell}{n}
\log_+\left(1+\frac{\sqrt{n}\hat{G}\max\{\bar{R},R_b\}}{\sigma \sqrt{\sum_{\ell = 1}^L m_{\ell+1} m_\ell}}\right)
+ 4\frac{r^2}{n}\right\}.
\]
From the optimality of $\hat{f}$ for the empirical risk, it follows that
\ba
&\frac{1}{n}\sum_{i=1}^n \{y_i - \hat{f}(x_i)\}^2
\le 
\frac{1}{n}\sum_{i=1}^n\{y_i - f^\ast(x_i)\}^2\\
\rightharpoonup\;\;
&
\|\hat{f}\|_n^2 - \|f^\ast\|_n^2
+\frac{2}{n} \sum_{i=1}^n y_i \{f^\ast(x_i)-\hat{f}(x_i)\}
\le 0\\
\rightharpoonup\;\;
&
\|\hat{f}\|_n^2 - \|f^\ast\|_n^2
+\frac{2}{n} \sum_{i=1}^n \{\xi_i + f^o(x_i)\} \{f^\ast(x_i)-\hat{f}(x_i)\}
\le 0\\
\rightharpoonup\;\;
&
\frac{2}{n} \sum_{i=1}^n \xi_i\{f^\ast(x_i)-\hat{f}(x_i)\}
+\|\hat{f}\|_n^2 + \|f^o\|_n^2 
+\frac{2}{n} \sum_{i=1}^n f^o(x_i)\{f^\ast(x_i)-\hat{f}(x_i)\}
\le  \|f^\ast\|_n^2 + \|f^o\|_n^2\\
\rightharpoonup\;\;
&
\frac{2}{n} \sum_{i=1}^n \xi_i\{f^\ast(x_i)-\hat{f}(x_i)\}
+\|\hat{f}-f^o\|_n^2
\le  \|f^\ast-f^o\|_n^2.
\ea
Hence, the inequality $(\ref{eq:complex-inequality1})$ implies
\ba
-\frac{1}{2}\left( \|\hat{f}-f^\ast\|_n \vee \sigma\sqrt{\frac{\sum_{\ell = 1}^L m_{\ell+1} m_\ell}{n}}\right)^2
-\Psi_{r,n}/2
+\|\hat{f}-f^o\|_n^2
\le  \|f^\ast-f^o\|_n^2.
\ea
If $\|\hat{f}-f^\ast\|_n \ge \sigma\sqrt{\sum_{\ell = 1}^L m_{\ell+1} m_\ell/n}$, 
it follows that
\ba
&-\frac{1}{2}\|\hat{f}-f^\ast\|_n^2
-\frac{1}{2}\Psi_{r,n}
+\|\hat{f}-f^o\|_n^2
\le  \|f^\ast-f^o\|_n^2\\
\rightharpoonup\;\;
&
-\frac{1}{2}\|\hat{f}-f^\ast\|_n^2
-\frac{1}{2}\Psi_{r,n}
+\frac{3}{4}\|\hat{f}-f^\ast\|_n - 3\|f^\ast-f^o\|_n
\le  \|f^\ast-f^o\|_n^2\\
\rightharpoonup\;\;
&
\frac{1}{4}\|\hat{f}-f^\ast\|_n
\le  4\|f^\ast-f^o\|_n^2+\frac{1}{2}\Psi_{r,n}.
\ea
Therefore, the inequality $(\ref{eq:complex-inequality1})$ gives
\banum
\|\hat{f}-f^\ast\|_n
\le  16\|f^\ast-f^o\|_n^2+2\Psi_{r,n} + \frac{\sigma^2\sum_{\ell = 1}^L m_{\ell+1} m_\ell}{n}.\label{eq:result-e.1}
\eanum

\subsection{Evaluation of $\|\hat{f}-f^\ast\|_{L_2(P_X)}$}

Based on the inequality $(\ref{eq:result-e.1})$, 
we will derive the upper bound for $\|\hat{f}-f^\ast\|_{L_2(P_X)}$.
Let $\Gcal_{\delta}^\prime=\{f-f^\ast \mid \|f-f^\ast\|_{L_2(P_X)}\le \delta ,f\in \Fcal\}$.
For $g\in \Gcal_\delta^\prime$, we have $\|g\|_\infty \le 2\hat{R}_\infty$, which gives
\[
\Eb[g(X_i)^4] = \int \{f(x) - f^\ast(x)\}^4 P_X(dx)\le 4\hat{R}_\infty^2\delta^2.
\]
Here, it follows that, for $g\in \Gcal_\delta^\prime$,
\[
\Eb[\{g(X_i)^2 - \Eb[g^2]\}^2]=\Eb[g(X_i)^4] - \{\Eb[g^2]\}^2\le \Eb[g(X_i)^4]\le 4\hat{R}_\infty^2\delta^2=:\tau^2
\]
and
\[
\|g^2 - \Eb[g^2]\|_\infty\le \sup_{x\in X}|g(x)^2-\Eb[g^2]| \le \max\{\sup_{x\in X}g(x)^2,\Eb[g^2]\}\le 4\hat{R}_\infty^2=:U.
\]
Set 
\[
S_n = \sup_{g\in \Gcal_\delta^\prime}\left| \sum_{i=1}^n \{ g(X_i)^2 - \Eb[g^2] \} \right|.
\]
By Bousquet's version of Talagrand's inequality (see, Theorem 3.3.9 in \cite{GineNickl15}),
it follows that
\ba
&\Pb\left( S_n \ge \Eb[S_n] + \sqrt{r(4U\Eb[S_n] +2n\tau^2)} + \frac{Ur}{3}\right)
\le \exp(-r).
\ea
Since
\ba
\sqrt{r(4U\Eb[S_n] +2n\tau^2)} \le 2\sqrt{rU\Eb[S_n])}+ \sqrt{2n\tau^2r}
\le rU + \Eb[S_n]+ \sqrt{2n\tau^2r},
\ea
we have
\ba
\Pb\left( S_n \ge 2\Eb[S_n]  + \sqrt{2n\tau^2}+ \frac{4Ur}{3}\right)
\le \exp(-r).
\ea
Combining these, we can see that there exists a universal constant $C$ such that
\banum
&\Pb\left( \sup_{g\in \Gcal_\delta^\prime}\left| \frac{1}{n}\sum_{i=1}^n \{ g(X_i)^2 - \Eb[g^2] \} \right| 
\ge C \left\{ 
\Eb\left[\sup_{g\in \Gcal_\delta^\prime}\left| \frac{1}{n}\sum_{i=1}^n \{ g(X_i)^2 - \Eb[g^2] \} \right|\right]  + \sqrt{\frac{\hat{R}_\infty^2\delta^2 r}{n}}+ \frac{\hat{R}_\infty^2 r}{n}
\right\}
\right)\nonumber\\
&\le \exp(-r).\label{eq:bound-theorem3.3.9}
\eanum

Now, we will consider the upper bound of $\Eb\left[\sup_{g\in \Gcal_\delta^\prime}\left| \sum_{i=1}^n \{ g(X_i)^2 - \Eb[g^2] \}/n \right|\right]$.
Let $\epsilon_1,\dots,\epsilon_n$ be an i.i.d. Rademacher sequence.
Then, by the usual result of Rademacher complexity (see, e.g., Lemma 2.3.1 in \cite{vanderVaartWellner96}),
it follows that
\ba
\Eb\left[\sup_{g\in \Gcal_\delta^\prime}\left| \frac{1}{n}\sum_{i=1}^n \{ g(X_i)^2 - \Eb[g^2] \} \right|\right] 
\le 2\Eb\left[\sup_{g\in \Gcal_\delta^\prime}\left| \frac{1}{n}\sum_{i=1}^n \epsilon_ig(X_i)^2 \right|\right] .
\ea
By the comparison theorem (Theorem 4.12 of \cite{LedouxTalagrand91}) with $\|g\|_\infty\le 2\hat{R}_\infty$, 
we have
\ba
2\Eb\left[\sup_{g\in \Gcal_\delta^\prime}\left| \frac{1}{n}\sum_{i=1}^n \epsilon_ig(X_i)^2 \right|\right] 
\le 
4(2\hat{R}_\infty) \Eb\left[\sup_{g\in \Gcal_\delta^\prime}\left| \frac{1}{n}\sum_{i=1}^n \epsilon_ig(X_i) \right|\right].
\ea
By using the derivation in the proof of Proposition 2.1 of \cite{GineGuillou01}, 
we derive the upper bound of the right hand side.
Write
\ba
Y:=\sup_{g \in \Gcal_\delta^\prime}\frac{1}{n}\sum_{i=1}^n g(X_i)^2.
\ea
Then, by Theorem 1.7 of \cite{Mendelson02} or Corollary 5.1.8 of \cite{delaPenaGine99}, 
we have
\[
T:=\frac{1}{\sqrt{n}} \Eb_\epsilon \left[ \sup_{g \in \Gcal_\delta^\prime}\left| \sum_{i=1}^n\epsilon_i g(X_i) \right| \right]
\le 
C \int_{0}^{\sqrt{Y}} \sqrt{\log\{N(\epsilon,\Gcal_\delta^\prime, \|\cdot\|_n)\}} \,\drm \epsilon,
\]
where $C$ is a universal constant.
From Proposition~\ref{prop:covering-number}, it follows that
 \ba
 \int_{0}^{\sqrt{Y}} \sqrt{\log\{N(\epsilon,\Gcal_\delta^\prime, \|\cdot\|_n)\}} \,\drm \epsilon
 &\le 
 C^\prime\sqrt{Y}\sqrt{\left\{\sum_{\ell = 1}^L (m_{\ell+1}+1) m_\ell\right\}\log_+\left(1+\frac{2\hat{G}\max\{\bar{R},R_b\}}{\sqrt{Y}}\right)},
 \ea
for another universal constant $C^\prime$.
By Jensen's inequality, we have
\ba
\Eb\left[ \sqrt{Y}\sqrt{\log_+\left(1+\frac{2\hat{G}\max\{\bar{R},R_b\}}{\sqrt{Y}}\right)}\right]
&\le 
2\sqrt{\Eb\left[Y\right]}\sqrt{\log_+\left(1+\frac{2\hat{G}\max\{\bar{R},R_b\}}{\sqrt{\Eb\left[Y\right]}}\right)}.
\ea
From Corollary 3.4 of \cite{Talagrand94}, it follows that
 \ba
 \Eb\left[ \sup_{g \in \Gcal_\delta^\prime}\sum_{i=1}^n g(X_i)^2 \right]
 \le n \delta^2 + 8(2\hat{R}_\infty)\Eb \left[ \sup_{g \in \Gcal_\delta^\prime}\left| \sum_{i=1}^n\epsilon_i g(X_i) \right| \right].
 \ea
 Thus, 
  \ba
 \Eb\left[ Y \right]
 \le \delta^2 + \frac{8(2\hat{R}_\infty)}{n}\Eb \left[ \sup_{g \in \Gcal_\delta^\prime}\left| \sum_{i=1}^n\epsilon_i g(X_i) \right| \right]
 = \delta^2 + \frac{8(2\hat{R}_\infty)}{\sqrt{n}}T.
 \ea
By the monotonicity of $\sqrt{x}\log(1+B/\sqrt{x})$, we have
 \ba
 \sqrt{\Eb\left[Y\right]}\sqrt{\log_+\left(1+\frac{2\hat{G}\max\{\bar{R},R_b\}}{\sqrt{\Eb\left[Y\right]}}\right)}
 &\le
  \sqrt{\delta^2 + \frac{8(2\hat{R}_\infty)}{\sqrt{n}}T}\sqrt{\log_+\left(1+\frac{2\hat{G}\max\{\bar{R},R_b\}}{\delta}\right)},
 \ea
 Let us introduce the temporary notation $A$ for $\sum_{\ell = 1}^L (m_{\ell+1}+1) m_\ell$.
Combining these, we get
 \ba
T:=\frac{1}{\sqrt{n}} \Eb_\epsilon \left[ \sup_{g \in \Gcal_\delta^\prime}\left| \sum_{i=1}^n\epsilon_i g(X_i) \right| \right]
&\le 
C \int_{0}^{\sqrt{Y}} \sqrt{\log\{N(\epsilon,\Gcal_\delta^\prime, \|\cdot\|_n)\}} \,\drm \epsilon\\
&\le
C \sqrt{\delta^2 + \frac{8(2\hat{R}_\infty)}{\sqrt{n}}T}\sqrt{A\log_+\left(1+\frac{2\hat{G}\max\{\bar{R},R_b\}}{\delta}\right)},
\ea
where $C$ is a universal constant. 
Therefore, we obtain
 \ba
 T^2 \le C A \delta^2\log_+\left(1+\frac{2\hat{G}\max\{\bar{R},R_b\}}{\delta}\right) +C\frac{8(2\hat{R}_\infty)}{\sqrt{n}}A\log_+\left(1+\frac{2\hat{G}\max\{\bar{R},R_b\}}{\delta}\right)T.
 \ea
A simple calculation leads
\ba
T \le\; 
&
C\Biggl\{
\hat{R}_\infty\frac{\sum_{\ell = 1}^L m_{\ell+1}m_\ell}{\sqrt{n}}\log_+\left(1+\frac{2\hat{G}\max\{\bar{R},R_b\}}{\delta}\right)\\
&\qquad\vee
\delta\sqrt{\left(\sum_{\ell = 1}^L m_{\ell+1}m_\ell\right) \log_+\left(1+\frac{2\hat{G}\max\{\bar{R},R_b\}}{\delta}\right)}
\Biggr\}.
\ea
Therefore, we can conclude that
\ba
\Eb\left[\sup_{g\in \Gcal_\delta^\prime}\left| \frac{1}{n}\sum_{i=1}^n \{ g(X_i)^2 - \Eb[g^2] \} \right|\right]
&\le 2\Eb\left[\sup_{g\in \Gcal_\delta^\prime}\left| \frac{1}{n}\sum_{i=1}^n \epsilon_ig(X_i)^2 \right|\right] 
\le 4(2\hat{R}_\infty) \Eb\left[\sup_{g\in \Gcal_\delta^\prime}\left| \frac{1}{n}\sum_{i=1}^n \epsilon_ig(X_i) \right|\right]\\
&\le
C\bigggl\{
\hat{R}_\infty^2\frac{\sum_{\ell = 1}^L m_{\ell+1}m_\ell}{n}\log_+\left(1+\frac{2\hat{G}\max\{\bar{R},R_b\}}{\delta}\right)\\
&\qquad\quad
\vee
\delta\hat{R}_\infty\sqrt{ \frac{\sum_{\ell = 1}^L m_{\ell+1}m_\ell}{n} \log_+\left(1+\frac{2\hat{G}\max\{\bar{R},R_b\}}{\delta}\right)}
\bigggr\}.
\ea

Let
\[
\Phi_n:= \frac{\sum_{\ell = 1}^L m_{\ell+1}m_\ell}{n}\log_+\left(1+\frac{2\sqrt{n}\hat{G}\max\{\bar{R},R_b\}}{\hat{R}_\infty \sqrt{\sum_{\ell = 1}^L m_{\ell+1}m_\ell}}\right).
\]
We apply the inequality $(\ref{eq:bound-theorem3.3.9})$ repeatedly for $\delta=2^{j-1}\hat{R}_\infty/\sqrt{n}\;\;(j=1,\dots,\lceil \log_2(\sqrt{n})\rceil+1)$.
Then, with probability at least $1- (\lceil \log_2(\sqrt{n})\rceil+1)\exp(-r)$, we have that
\ba
&\left|
\frac{1}{n}\sum_{i=1}^n\{\hat{f}(X_i) - f^\ast(X_i)\}^2 - \Eb[\{\hat{f}(X)-f^\ast(X)\}^2]
\right|\\
&\le 
C_1 \left\{ 
C\left(\hat{R}_\infty^2\Phi_n\vee \delta(\hat{f}) \hat{R}_\infty\sqrt{\Phi_n} \right)  + \sqrt{\frac{\hat{R}_\infty^2\delta^2 r}{n}}+ \frac{\hat{R}_\infty^2 r}{n}
\right\}\\
&\le
\frac{1}{2}\max\left(\|\hat{f} - f^\ast\|_{L_2(P_X)}^2, \hat{R}_\infty^2 \frac{\sum_{\ell=1}^Lm_{\ell +1}m_\ell}{n}\right) + 
C_2\hat{R}_\infty^2\left( \Phi_n  +  \frac{r}{n} \right),
\ea
where $C, C_1$, and $C_2$ are universal constants, and 
\[
\delta(\hat{f})=\max\left(\|\hat{f} - f^\ast\|_{L_2(P_X)}, \hat{R}_\infty\sqrt{ \frac{\sum_{\ell=1}^Lm_{\ell +1}m_\ell}{n}}\right).
\]
Combining this inequality with the inequality $(\ref{eq:result-e.1})$, we deduce that
\banum
\frac{1}{2}\|\hat{f}-f^\ast\|_{L_2(P_X)}^2
&\le 
16\|f^\ast-f^o\|_n^2+2\Psi_{r,n} + \frac{\sigma^2 + \hat{R}_\infty^2}{n}\sum_{\ell = 1}^L m_{\ell+1} m_\ell
 + C \hat{R}_\infty^2\left( \Phi_n  +  \frac{r}{n} \right),\label{eq:result-e.2}
\eanum
where $C$ is a universal constant.

\subsection{Evaluation of $\|f^\ast - f^o\|_n^2$}

In the inequality $(\ref{eq:result-e.2})$, it remains to be clarified the upper bound of $\|f^\ast - f^o\|_n^2$.
Here, we note that
\[
\|f^\ast - f^o\|_n^2 -\|f^\ast - f^o\|_{L_2(P_X)}  = \frac{1}{n}\sum_{i=1}^n \left\{ (f^\ast(X_i) - f^o(X_i))^2 - \Eb\left[ \{f^\ast(X_i) - f^o(X_i)\}^2\right] \right\}
\]
and that
\[
\left| (f^\ast(X_i) - f^o(X_i))^2 - \Eb\left[ \{f^\ast(X_i) - f^o(X_i)\}^2\right]\right|
\le \|f^\ast - f^o\|_\infty^2.
\]
By the Bernstein's inequality, we have
\ba
\Pb\left(\|f^\ast - f^o\|_n^2 \ge \|f^\ast - f^o\|_{L_2(P_X)} +  t \right)
\le 
\exp\left(-\frac{t^2n}{2(v + t\|f^\ast - f^o \|_\infty/3)} \right),
\ea
where
\ba
v &:= \Eb\left[ \left\{ (f^\ast(X_i) - f^o(X_i))^2 - \Eb\left[ \{f^\ast(X_i) - f^o(X_i)\}^2\right] \right\}^2 \right].
\ea
From Theorem~\ref{theorem:approx-error}, it follows that
\ba
v \le
\Eb\left[ (f^\ast(X_i) - f^o(X_i))^4\right]
\le 
\|f^\ast - f^o\|_\infty^2\|f^\ast - f^o\|_{L_2(P_X)}^2\le \|f^\ast - f^o\|_\infty^2\hat{\delta}_{1,n}^2.
\ea
Thus, substituting $t \leftarrow \tilde{r} \times \|f^\ast - f^o\|_{L_2(P_X)}^2$ for $\tilde{r} \in (0,1]$, 
we obtain 
\ba
\Pb\left(\|f^\ast - f^o\|_n^2 \ge (1+\tilde{r})\hat{\delta}_{1,n}^2 \right)
&\le
\Pb\left(\|f^\ast - f^o\|_n^2 \ge \|f^\ast - f^o\|_{L_2(P_X)} + \tilde{r}\hat{\delta}_{1,n}^2\right)\\
&\le 
\exp\left(-\frac{\tilde{r}^2\hat{\delta}_{1,n}^4n}{2\|f^\ast - f^o \|_\infty \left(\|f^\ast - f^o\|_{L_2(P_X)}^2 + \tilde{r}\hat{\delta}_{1,n}^2/3\right)} \right)\\
&\le
\exp\left(-\frac{3n\hat{\delta}_{1,n}^2}{8}\frac{\tilde{r}^2}{\|f^\ast - f^o \|_\infty} \right)
\le
\exp\left(-\frac{3n\hat{\delta}_{1,n}^2\tilde{r}^2}{32} \right).
\ea
Therefore, for any $\tilde{r} \in (0,1]$, with probability at least $1-\exp\left(-3n\hat{\delta}_{1,n}^2\tilde{r}^2/32 \right)$, we have 
\banum
\|f^\ast - f^o\|_n^2 \le (1+\tilde{r})\hat{\delta}_{1,n}^2.\label{eq:result-e.3}
\eanum

\subsection{Evaluation of $\|\hat{f}-f^o\|_{L_2(P_X)}^2$ (Final step)}

Combining $(\ref{eq:result-e.2})$ with $(\ref{eq:result-e.3})$, we get
\ba
\|\hat{f}-f^\ast\|_{L_2(P_X)}^2
&\le 
32(1+\tilde{r})\hat{\delta}_{1,n}^2+4\Psi_{r,n} + \frac{2(\sigma^2 + \hat{R}_\infty^2)}{n}\sum_{\ell = 1}^L m_{\ell+1} m_\ell
 + 2C\hat{R}_\infty^2\left( \Phi_n  +  \frac{r}{n} \right).
\ea
Since 
\ba
\|\hat{f}-f^o\|_{L_s(P_X)}^2
&\le 
2\left\{ \|\hat{f}-f^\ast\|_{L_s(P_X)}^2+\|f^\ast - f^o\|_{L_s(P_X)}^2\right\}
\le
2\|\hat{f}-f^\ast\|_{L_s(P_X)}^2 + 2\hat{\delta}_{1,n},
\ea
we have 
\ba
\|\hat{f}-f^\ast\|_{L_2(P_X)}^2
&\le 
(66+64\tilde{r})\hat{\delta}_{1,n}^2+8\Psi_{r_1,n} + \frac{4(\sigma^2 + \hat{R}_\infty^2)}{n}\sum_{\ell = 1}^L m_{\ell+1} m_\ell
 + 4C \hat{R}_\infty^2\left( \Phi_n  +  \frac{r_2}{n} \right)
\ea
with probability at least
\[
1-\exp\left(-3n\hat{\delta}_{1,n}^2\tilde{r}^2/32 \right)-2\log_2(\sqrt{n})\exp(-r_2)-2\log_2(\hat{R}_\infty\sqrt{n}/\sigma) \exp(-r_1^2/2).
\]
Let us introduce the temporary notation
\[
\alpha(x) := x^2\frac{\sum_{\ell = 1}^L m_{\ell+1} m_\ell}{n}
\log_+\left(1+\frac{\sqrt{n}\hat{G}\max\{\bar{R},R_b\}}{x \sqrt{\sum_{\ell = 1}^L m_{\ell+1} m_\ell}}\right)\quad(x\ge 0).
\]
For $r>0$, write
\ba
r_{1}=\sqrt{2\left\{ \log(2\log_2(\hat{R}_\infty\sqrt{n}/\sigma)) + r\right\}},\quad
r_{2}=\log(2\log_2(\sqrt{n})) + r.
\ea
Then, it follows that
\ba
\Psi_{r_1,n} 
&= C \sigma^2\left\{\frac{\sum_{\ell = 1}^L m_{\ell+1} m_\ell}{n}
\log_+\left(1+\frac{\sqrt{n}\hat{G}\max\{\bar{R},R_b\}}{\sigma \sqrt{\sum_{\ell = 1}^L m_{\ell+1} m_\ell}}\right)
+ 8\frac{\log(2\log_2(\hat{R}_\infty\sqrt{n}/\sigma)) + r}{n}\right\}\\
&\le 
C
\left\{
\alpha(\sigma) + \frac{\sigma^2}{n}\log_+\left(\frac{\sqrt{n}}{\sigma/\hat{R}_\infty}\right) + \sigma^2\frac{r}{n}
\right\}
\ea
and that
\ba
\hat{R}_\infty^2\left( \Phi_n  +  \frac{r_2}{n} \right)
&\le 
\hat{R}_\infty^2\frac{\sum_{\ell = 1}^L m_{\ell+1}m_\ell}{n}\log_+\left(1+\frac{2\sqrt{n}\hat{G}\max\{\bar{R},R_b\}}{\hat{R}_\infty \sqrt{\sum_{\ell = 1}^L m_{\ell+1}m_\ell}}\right)
+
\frac{\hat{R}_\infty^2}{n}\left( \log(2\log_2(\sqrt{n})) + r\right)\\
&\le
Const.\times \left\{ \alpha(\hat{R}_\infty)+\frac{\hat{R}_\infty^2}{n} \log_+(\sqrt{n}) + \frac{\hat{R}_\infty^2}{n}r \right\}.
\ea
Accordingly, we have
\ba
\|\hat{f}-f^o\|_{L_2(P_X)}^2
&\le 
C
\left\{
\alpha(\sigma) + \alpha(\hat{R}_\infty) + \frac{\sigma^2 
+ \hat{R}_\infty^2}{n}\left[ \log_+\left(\frac{\sqrt{n}}{\min\{1,\sigma/\hat{R}_\infty\}}\right)  + r\right]
+(1+\tilde{r})\hat{\delta}_{1,n}^2
\right\}
\ea
with probability at least
\[
1-\exp\left(-\frac{3n\hat{\delta}_{1,n}^2\tilde{r}^2}{32} \right)-2\exp(-r).
\]
Let 
\[
\hat{\delta}_{2,n}:=
\sqrt{\frac{\sum_{\ell = 1}^L m_{\ell+1} m_\ell}{n}
\log_+\left(1+\frac{\sqrt{n}\hat{G}\max\{\bar{R},R_b\}}{\min\{\sigma, \hat{R}_\infty\} \sqrt{\sum_{\ell = 1}^L m_{\ell+1} m_\ell}}\right)}.
\]
For any $r>0$ and any $\tilde{r}\in(0,1]$, 
with probability at least
\[
1-\exp\left(-\frac{3n\hat{\delta}_{1,n}^2\tilde{r}^2}{32} \right)-2\exp(-r),
\]
we have
\ba
\|\hat{f}-f^o\|_{L_2(P_X)}^2
&\le 
C
\left\{(1+\tilde{r})\hat{\delta}_{1,n}^2+
(\sigma^2+ \hat{R}_\infty^2)\hat{\delta}_{2,n}^2 + \frac{\sigma^2 
+ \hat{R}_\infty^2}{n}\left[ \log_+\left(\frac{\sqrt{n}}{\min\{1,\sigma/\hat{R}_\infty\}}\right)  + r\right]
\right\}.
\ea
This is our claim.
\end{document}